\documentclass[journal]{IEEEtran}
\IEEEoverridecommandlockouts
\usepackage{cite}
\usepackage{amsmath,amssymb,amsfonts,bm,amsthm}
\usepackage{autobreak}
\usepackage{algorithmic,algorithm}
\usepackage{graphicx}
\usepackage{textcomp}
\usepackage{xcolor}
\usepackage{epstopdf}
\usepackage{url}
\usepackage{subfigure}
\newtheorem{theorem}{Theorem}
\newtheorem{lemma}{Lemma}
\newtheorem{remark}{Remark}

\newtheorem{definition}{Definition}

\def\BibTeX{{\rm B\kern-.05em{\sc i\kern-.025em b}\kern-.08em
    T\kern-.1667em\lower.7ex\hbox{E}\kern-.125emX}}
\begin{document}

\title{Differentially Private ADMM for Distributed Medical Machine Learning}

\author{
\IEEEauthorblockN{Jiahao Ding,~\IEEEmembership{Student Member,~IEEE}, Xiaoqi Qin,~\IEEEmembership{Member,~IEEE}, Wenjun Xu,~\IEEEmembership{Senior Member,~IEEE}, Yanmin Gong,~\IEEEmembership{Member,~IEEE}, Chi Zhang,~\IEEEmembership{Senior Member,~IEEE} and Miao Pan,~\IEEEmembership{Senior Member,~IEEE}}

\thanks{J. Ding and M. Pan are with the Department of Electrical and Computer Engineering, University of Houston, Houston, TX 77204. (Email: jding7@uh.edu, mpan2@uh.edu ). 
X. Qin and W. Xu are with the School of Information and Communication Engineering, Beijing University of Posts and Telecommunications, Beijing, China 100876. (Email: xiaoqiqin@bupt.edu.cn, wjxu@bupt.edu.cn). 
Y. Gong is with the Department of Electrical and Computer Engineering, University of Texas at San Antonio, TX 78249. (Email: yanmin.gong@utsa.edu). 
C. Zhang is with School of Information Science and Technology, University of Science and Technology of China, Hefei, China 230027. (Email: chizhang@ustc.edu.cn).}
}

\maketitle

\begin{abstract}
Due to massive amounts of data
distributed across multiple locations,  distributed machine learning has attracted a lot of research interests. Alternating Direction Method of Multipliers (ADMM) is a powerful method of designing distributed machine learning algorithm, whereby each agent computes over local datasets and exchanges computation results with its neighbor agents in an iterative procedure. There exists significant privacy leakage during this iterative process if the local data is sensitive. In this paper, we propose a differentially private ADMM algorithm (P-ADMM) to provide dynamic zero-concentrated differential privacy (dynamic zCDP), by inserting Gaussian noise with linearly decaying variance. We prove that P-ADMM has the same convergence rate compared to the non-private counterpart, i.e., $\mathcal{O}(1/K)$ with $K$ being the number of iterations and linear convergence for general convex and strongly convex problems while providing differentially private guarantee. Moreover, through our experiments performed on real-world datasets, we empirically show that P-ADMM has the best-known performance among the existing differentially private ADMM based algorithms.
\end{abstract}

\section{Introduction}
With the rapid development of sensing technologies, the vast amount of data has already been generated in recent years. Building machine learning models over such deluge data in one machine is impossible, especially in many cases that data is distributed across different locations. Therefore, the advent of distributed machine learning is imperative to reduce storage and local computational costs, while improving robustness of machine learning models~\cite{xing2016strategies}. 
Several distributed optimization approaches have been developed to optimize large-scale machine learning problems. In particular, as a popular optimization technique, alternating direction method of multipliers (ADMM)\cite{M2AN_1975__9_2_41_0}\cite{gabay1976dual} has been very often used in the context of distributed machine learning \cite{taylor2016training,zheng2016fast, zhang2018training}, which is extremely robust and easy to implement. Furthermore, ADMM based algorithms can typically achieve a fast convergence rate $\mathcal{O}(1/K)$, where $K$ is the number of iterations~\cite{wei2012distributed}.

In the distributed setting, the agents on a connected network process data locally by minimizing small optimization problems,
and exchange the local results among the neighbors to arrive at a global minimization solution.
However, the communication among agents in such a distributed manner brings about privacy concerns if the local agent's training data contains sensitive information such as salary or medical records. Therefore, there is a great need of privacy preserving algorithm to ensure privacy guarantee in such iterative processes. Differential privacy \cite{dwork2014algorithmic} is a well-defined framework for quantifying privacy risk revealed by an algorithm from adversarial inference. Briefly, differentially private algorithms prevent the adversary with arbitrary background knowledge by obfuscating the query responses. 
In the literature, there are some pioneering work on integrating differential privacy and ADMM \cite{zhang2016dual, zhang2018improving,zhang2018recycled, huang2018dp}.
In \cite{zhang2016dual}, Zhang and Zhu proposed two privacy preserving ADMM algorithms by perturbing the dual variable and the primal variable, respectively; However, both of algorithms just bounded the privacy loss of one agent at a single iteration. Later, in \cite{zhang2018improving} a penalty perturbation method to bound the total privacy loss of all agents instead of only one agent during the entire process is developed. Moreover, in \cite{zhang2018recycled} based on the same penalty perturbation method, Zhang et al. presented a modified ADMM algorithm by repeatedly using the existing computational results to reduce the computational burden and reach a less accumulated privacy loss. In addition, in \cite{huang2018dp}, Huang et al. employed an augmented Lagrangian function with first order approximation at all iterations, while they consider a network with star topology, i.e., a central server is required which summarizes all shared primal variables and broadcasts it back to all agents over the network.

In this paper, we propose a differentially private ADMM algorithm (P-ADMM) to remedy privacy concerns of distributed machine learning. Specifically, a linearly decaying Gaussian noise is added to perturb exchanged variables at each iteration. Then a new concept to quantify privacy leakage of a general distributed algorithm, called dynamic zero-concentrated differential privacy (dynamic zCDP), is developed, which stems from the definition of zero-concentrated differential privacy \cite{bun2016concentrated} and enjoys significant accuracy and tight privacy loss composition. Furthermore, we theoretically prove the convergence rate of P-ADMM on general convex and strongly convex problems and provide a rigorous privacy analysis. Overall, the main contributions of this paper are summarized as follows.
\begin{itemize}
    \item We propose a differentially private ADMM algorithm (P-ADMM) by introducing a Gaussian noise with a linearly decaying variance to address the privacy concerns in distributed machine learning over large datasets.
    \item We introduce a new privacy framework to quantify the privacy leakage in a distributed and iterative setting, called dynamic zCDP, and present the privacy analysis of P-ADMM based on this privacy framework.
    \item We provide convergence analysis of P-ADMM for general convex and strongly convex objectives. Note that our analysis shows that P-ADMM has a linear convergence rate for strongly convex problems. For general convex problems, P-ADMM shows an $\mathcal{O}(1/K)$ convergence rate, where $K$ is the iteration number.
    \item Using real-world datasets, we demonstrate that the algorithmic performance of P-ADMM significantly outperforms the state-of-the-art differentially private ADMM based algorithms. Specifically, P-ADMM exhibits nearly the same convergence properties as the non-private version while preserving differential privacy.
\end{itemize}

The rest of paper is organized as follows.
Section~\ref{Problem formulation} gives the problem formulation, and definition of ADMM and differential privacy. Then, the definition of dynamic zCDP and a privacy preserving version of ADMM are presented in Section~\ref{Private ADMM}. In Section~\ref{convergence analysis}, we provide the convergence analysis of the proposed algorithm and the numerical experiments based on real-world datasets are shown in Section~\ref{experiments}. Finally, we conclude our work in Section~\ref{conclusion}.
\section{Problem Setting and Preliminaries}\label{Problem formulation}
\subsection{Problem Setting}
Throughout the paper, we consider a bidirectional network given by an undirected graph
$\mathcal{G} =\{\mathcal{N}, \mathcal{E}\}$, which contains the set of agents/nodes $\mathcal{N} = \{ 1,\cdots, N \}$ and the set of edges $\mathcal{E}$ with $|\mathcal{E}| = E$. The total number of communication links among connected agents in $\mathcal{E}$ is denoted as $2E$. Let $\mathcal{V}_i$ denote the set of neighboring agents of agent $i$. For any agent $i$, it can only exchange information among its neighbors in $\mathcal{V}_i$. Here, each agent $i$ contains a dataset $D_i = \{(y_i^n, z_i^n)\}_{n = 1}^{|D_i|}$, where $y_i^n \in \mathcal{C}^*$ is a feature vector and the corresponding label is $z_i^n \in \mathbb{R}^z$. The goal of our problem is to train a classifier $w \in \mathbb{R}^d$ over the union datasets $\hat{D}  = \cup_{i \in\mathcal{N} } D_i$ in a decentralized manner (i.e., no centralized controller/fusion center). In order to train this classifier, we model the problem in the following regularized empirical risk minimization (ERM) problem
\begin{align}\label{erm}
    \min_{w \in \mathbb{R}^d} \sum_{i=1}^N \frac{1}{|D_i|}\sum_{n = 1}^{|D_i|} \mathcal{L}(z_i^n, y_i^n, w) +\mathcal{R}(w),
\end{align}
where $\mathcal{L}(\cdot):\mathbb{R}^z \times \mathcal{C}^* \times \mathbb{R}^d \to \mathbb{R} $ is the loss function, and $\mathcal{R}(w): \mathbb{R}^d \to \mathbb{R}$ is the regularizer to prevent overfitting. 

Empirical risk minimization as a supervised learning problem is arising very often in machine learning~\cite{shalev2017understanding}. Examples of ERM consists of the widely applicable problems of classification and regression in machine learning. For instance, in binary classification problems ($z_i^n \in \{+1,-1\} $), we get linear support vector machine (SVM) and logistic regression by setting the loss function $\mathcal{L}(z_i^n, y_i^n, w) = \max \{1-z_i^n w^T y_i^n,0 \}$ and $\mathcal{L}(z_i^n, y_i^n, w) = \log(1+\exp(-z_i^n w^T y_i^n))$, respectively. Moreover, a further introduction of ERM in machine learning can also be found in \cite{trevor2009elements}\cite{zhang2018mixup}.
\subsection{Preliminaries}
\subsubsection{Conventional ADMM}
To decentralize ERM problem~(\ref{erm}), we introduce a variable $x_i \in\mathbb{R}^d $ for agent $i$. Then, the ERM problem~(\ref{erm}) can be formulated as follows
\begin{equation} \label{eq:nc}
\begin{array}{cl}
\min\limits_{\{x_i\},\{p_{ij}\}} &\sum\limits_{i=1}^N f_i(x_i) \\
\mbox{s.t.}\ &x_i = p_{ij},~ x_j = p_{ij}, ~\forall (i,j) \in \mathcal{E}, \\
\end{array}
\end{equation}
where $f_i(x_i) = \frac{1}{|D_i|}\sum_{n = 1}^{|D_i|} \mathcal{L}(z_i^n, y_i^n, x_i) + \frac{1}{N} \mathcal{R}(x_i)$\cite{4407653}.
Note that the objective function of (\ref{eq:nc}) can be easily decoupled across the agents, and then each agent $i$ obtains a local classifier $x_i$ by only minimizing objective $f_i(x_i)$ over its own dataset. Moreover, since the network is connected, the constraints $x_i = p_{ij}$ and $ x_j = p_{ij}$ enforce global consensus, i.e., all local classifiers should be equal. Therefore, (\ref{eq:nc}) is equivalent to (\ref{erm}).
The ADMM algorithm can intuitively solve the above optimization problem in (\ref{eq:nc}) in a decentralized and collaborative manner.

To specify the ADMM algorithm, we rewrite problem (\ref{eq:nc}) in a matrix form. Defining $x:=[x_1;\cdots;x_N] \in \mathbb{R}^{Nd}$ and $p \in \mathbb{R}^{2Ed}$ as a vector concatenating all local variables $x_i$ and $p_{ij}$ separately. Further, define two blocks $A_1$ and $A_2\in \mathbb{R}^{2Ed \times Nd}$ that are consisted of $2E \times N$ blocks of $d\times d$ matrices. If $p_{ij}$ is the $q$-th block of $p$ and $(i,j) \in \mathcal{E}$, then both of the $(q,i)$-th block of $A_1$ and $(q,j)$-th block of $A_2$ are equal to identity matrices $I_d$. Otherwise the blocks are null. As a result, the matrix form of problem in (\ref{eq:nc}) can be written as follows
\begin{equation} \label{eq:nc-mat}
\begin{array}{cl}
\min\limits_{x,p} &f(x) + g(p), \\
\mbox{s.t.} &Ax + Bp = 0, \\
\end{array}
\end{equation}
where matrices $A := [A_1;A_2]$, $B: = [-I_{2Ed};-I_{2Ed} ]$, and aggregate function $f: \mathbb{R}^{Nd}\to \mathbb{R}$ is defined as $f(x)=\sum_{i=1}^N f_i(x_i)$, and $g(p) = 0$, which fits the standard form of ADMM.   

By assigning Lagrangian multipliers $\lambda \in \mathbb{R}^{4Ed}$, the augmented Lagrangian function of (\ref{eq:nc-mat}) is given by 
\begin{equation}\label{lag}
    L_c(x,p,\lambda) = f(x) + \left<Ax+Bp, \lambda\right> + \frac{\eta}{2}\|Ax+Bp\|_2^2.
\end{equation}
The idea of ADMM algorithm is to alternatively minimize $L_c(x,p,\lambda)$ regarding variables $x$, $p$ and $\lambda$. At a specific iteration $k+1$, ADMM has the following steps
\begin{align}
      \nabla f(x^{k+1}) + A^T\lambda^k + \eta A^T(Ax^{k+1}+Bp^k)&=0,\\
       B^T \lambda^k +\eta B^T(Ax^{k+1}+Bp^{k+1}) &=0, \\
       \lambda^{k+1} - \lambda^k - \eta(Ax^{k+1}+Bp^{k+1}) &= 0.
\end{align}
If we let $\lambda = [\beta; \gamma]$ with $\beta, \gamma \in \mathbb{R}^{2EN}$, $H_+ = A_1^T + A_2^T$ and $H_- = A_1^T - A_2^T$, the above ADMM algorithm updates can be reduced to
\begin{align}
    \nabla f(x^{k+1}) + \alpha^k +2\eta Mx^{k+1} - \eta L_{+}x^k  &=0 \label{u-x},\\
    \alpha^{k+1}-\alpha^{k}-\eta L_{-}x^{k+1} &= 0 \label{u-a},
\end{align}
where $\alpha = H_-\beta \in \mathbb{R}^{Nd}$ is a new multiplier, and $M = \frac{1}{2}(L_+ + L_-)$ with $L_+ = \frac{1}{2}H_+H_+^T$ and $L_- = \frac{1}{2}H_-H_-^T$. 
Note that $L_+$ is the signless Laplacian matrix and $L_-$ is the signed Laplacian matrix of the network. 

Remember that the ADMM algorithm updates (\ref{u-x}) and (\ref{u-a}) are in matrix form. Hence, each agent $i$ uses the following updates to obtain its own variable $x_i$,
\begin{align}
    \nabla f_i(x_i^{k+1}) + \alpha_i^k +2\eta |\mathcal{V}_i|x_i^{k+1}   =\eta\left(|\mathcal{V}_i|x_i^k + \sum_{j\in \mathcal{V}_i }x_j^k\right), \label{u-xx}\\
    \alpha_i^{k+1}=\alpha_i^{k}+\eta \left(|\mathcal{V}_i|x_i^{k+1} - \sum_{j\in \mathcal{V}_i }x_j^{k+1}\right), \label{u-aa}
\end{align}
where $\alpha_i \in \mathbb{R}^d$ is the local Lagrange multiplier of agent $i$ and $\alpha$ is the concatenated form of all $\alpha_i$. 

At iteration $k+1$, agent $i$ updates $x_i^{k+1}$ through (\ref{u-xx}) using its own previous results $x_i^k$, $\alpha_i^k$ and its neighbors' previous results $x_j^k$ with $j\in \mathcal{V}_i$, and then broadcasts $x_i^{k+1}$ to all its neighbors. After collecting all of its neighbors results $x_j^{k+1}$, agent $i$ updates its local multiplier $\alpha_i$ through (\ref{u-aa}).

\subsubsection{Differential Privacy}
Differential privacy (DP) introduced in~\cite{dwork2014algorithmic}\cite{ zhangpri} provides rigorous privacy guarantees by injecting random noise to perturb the released statistical results obtained from sensitive datasets. It ensures that the adversary with enough background knowledge cannot infer any information about a specific record with high confidence. The definition of differential privacy is defined as follows.
\begin{definition}[\bf{Differential Privacy}]
We say a randomized algorithm $\mathcal{M}$ gives $(\epsilon,\delta)$-differential privacy if for any datasets $D$ and $\hat{D}$ that differ in at most a single record, the privacy loss random variable of an output $o \in Range(\mathcal{M})$
\begin{align*}
    Z(o) = \ln\frac{{\rm Pr}[\mathcal{M}(D)= o] }{{\rm Pr}[\mathcal{M}(\hat{D}) = o]}
\end{align*}
is bounded by $\epsilon$ with probability at least $1-\delta$.
\end{definition}
The Gaussian mechanism is a generic method for satisfying $(\epsilon,\delta)$-differential privacy by adding Gaussian noise calibrated to the query function's sensitivity. Then, we define the sensitivity as follows.
\begin{definition}[\bf{Sensitivity}]\label{sens}
The sensitivity of a query function $f(\cdot)$ that takes as input a dataset $D$ is defined as
\begin{align*}
    \Delta_f = \max_{D,\hat{D}} \|f(D)-f(\hat{D})\|_2,
\end{align*}
where $D$ and $\hat{D}$ are any two neighboring datasets differing in at most one record.
\end{definition} 
Based on the definition of sensitivity, we show the Gaussian mechanism in the following theorem.
\begin{theorem}[\bf{Gaussian Mechanism}]\label{gaussian}
For a query function $f:\mathcal{D} \to \mathcal{R}^d $ with sensitivity $\Delta_f$, the Gaussian Mechanism that adds noise generated from the Gaussian distribution $\mathcal{N}(0,\sigma^2 I_d)$ to the output of function $f$ enjoys $(\epsilon,\delta)$-differential privacy,
where $\epsilon,\delta \in (0,1)$ and $\sigma \geq \frac{\sqrt{2 \ln(1.25/\delta) \Delta_f}}{\epsilon}$.
\end{theorem}

A relaxed version of differential privacy called zero-concentrated differential privacy (zCDP) is proposed in \cite{bun2016concentrated}, which aims to make significantly
tighter privacy bounds for privacy preserving iterative algorithms.
The definition of $\rho$-zCDP is defined as follows.
\begin{definition}[\bf{$\rho$-zCDP}]
For all $\tau \in (1,\infty)$, a randomized algorithm $\mathcal{M}$ is $\rho$-zCDP if for all neighboring datasets $D$ and $\hat{D}$, we have 
\begin{align*}
    \mathbb{E}[e^{(\tau -1)Z(o)}] \leq e^{(\tau -1)\tau \rho},
\end{align*}
where $Z{(o)}$ is the privacy loss variable of an outcome of $\mathcal{M}$.
\end{definition}
The following three lemmas~\cite{bun2016concentrated} shows the Gaussian mechanism satisfies zCDP and the composition theorem of $\rho$-zCDP and the relationship between $\rho$-zCDP and $(\epsilon,\delta)$-DP.
\begin{lemma}
The Gaussian mechanism, defined in Definition~\ref{gaussian}, satisfies $ \Delta_f^2/ (2\sigma^2) $-zCDP.
\end{lemma}
\begin{lemma}\label{comp}
If randomized mechanisms $\mathcal{M}_1,..., \mathcal{M}_k$ satisfies $\rho_1$-zCDP,...,$\rho_k$-zCDP, their composition defined as $(\mathcal{M}_1,...,\mathcal{M}_k)$ is $\sum_{i =1}^k\rho_k$-zCDP.
\end{lemma}
\begin{lemma}\label{cdptodp}
If a randomized mechanism $\mathcal{M}$ provides $\rho$-zCDP, then for any $\delta \in (0,1)$ $\mathcal{M}$ satisfies $(\rho + 2\sqrt{\rho \ln(1/\delta)}, \delta)$-differential privacy.
\end{lemma}
\subsection{Privacy Concerns}
During the iterative process, there is no need for each agent to share its own dataset, but the risk of information leakage still exists if the local datasets contain sensitive information like medical or financial records. We assume the adversary's goal is to learn the private information of training datasets. The adversary is able to access all shared variables $\{x_i^k\}_{i=1}^N$ by eavesdropping the communications among agents and also have arbitrary auxiliary information. After obtaining exchanged variables during iterative process, the adversary can perform attacks (e.g., model inversion attack \cite{fredrikson2015model}) to infer private information of the training datasets. 

\section{Differentially Private ADMM}\label{Private ADMM}

To mitigate the privacy risk, we propose a new relaxed version of dynamic differential privacy~\cite{zhang2016dual}, called dynamic zero-concentrated differential privacy (dynamic zCDP) to provide privacy guarantee of the iterative algorithm. 
\begin{algorithm}
\caption{Differentially Private ADMM (P-ADMM) }
\label{alg:PADMM}
\algsetup{indent=2em}
\begin{algorithmic}[1]
 
\STATE \textbf{Input:} datasets $\{D_i\}_{i=1}^N$; randomly initialize $x_i^0$ and $\alpha_i^0 = 0_{d}$; initial variances ${(\sigma^2)}_i^{1}$ and variance decrease rate $R\in(0,1)$ for all agents $i$;
\FOR {$k = 0,1,\cdots,K-1$}
\FOR {$i = 1,\cdots,N$}
\STATE  Compute $x_i^{k+1}$ via (\ref{u-xx}).
\STATE  Generate noise $\xi_i^{k+1}\sim \mathcal{N}(0, {(\sigma^2)}_i^{k+1} I_{d})$.
\STATE  Perturb $x_i^{k+1}$: $\tilde{x}_i^{k+1} = x_i^{k+1} +\xi_i^{k+1}$.
\ENDFOR
\FOR {$i = 1,\cdots,N$} 
\STATE Broadcast $ \tilde{x}_i^{k+1}$ to all neighbors $j\in \mathcal{V}_i$.
\ENDFOR
\FOR{$i = 1,\cdots,N$}
\STATE Compute $\alpha_i^{k+1}$ via (\ref{ua}).
\ENDFOR 
\ENDFOR
\STATE \textbf{Output:} $\{\tilde{x}_i^K\}_{i=1}^N$
\end{algorithmic}
\end{algorithm}
\begin{definition}[\bf{Dynamic $\rho^k$-zCDP}]
Consider a connected network $\mathcal{G} =\{\mathcal{N}, \mathcal{E}\}$ that contains a set of agents/nodes $\mathcal{N} = \{ 1,\cdots, N \}$ and each agent possesses a training dataset $D_i$. We denote $x_i^k$ as the result of the $i$-th agent at $k$-iteration. A distributed algorithm gives dynamic $\rho_i^k$-zCDP if for all datasets $D_i$ and $\hat{D}_i$ differing at most a single record, and for all agents $i \in \mathcal{N}$, and for all $k$ during a learning process, the privacy loss variable of an outcome $o$
\begin{align*}
    Z_i^{k}(o) = \ln\frac{{\rm Pr}[x_{i,D_i}^k= o  ] }{{\rm Pr}[x_{i,\hat{D}_i}^k= o ]}
\end{align*}
satisfies 
\begin{align*}
    \mathbb{E}[e^{(\tau -1)Z_i^{k}(o)}] \leq e^{(\tau -1)\tau \rho_i^k},
\end{align*}
$\forall \tau \in (1,\infty)$.
\end{definition}

Compared with dynamic differential privacy proposed in~\cite{zhang2016dual}, dynamic zero-concentrated differential privacy directly inherits desirable properties of zero-concentrated differential privacy such as enjoying better accuracy and providing significantly tighter privacy loss bound under composition.

Based on the definition of dynamic zCDP, we propose a differentially private ADMM algorithm (P-ADMM) by injecting Gaussian noise in the iterative algorithm. As shown in Algorithm~\ref{alg:PADMM}, the main idea of P-ADMM is to perturb the local primal variables $\{x_i^{k+1}\}_{i =1}^N$ before broadcasting them to the neighbors. The whole procedure is summarized as follows: at iteration $k+1$, each agent $i$ obtains the primal variable $x_i^{k+1}$ by solving the subproblem in~(\ref{u-xx}) over its own dataset $D_i$. Then, each agent generates a random noise vector $\xi_i^{k+1}$ drawn from a Gaussian distribution $\mathcal{N}(0, {(\sigma^2)}_i^{k+1} I_{d})$. After that, the primal variable $x_i^{k+1}$ is perturbed by the noise vector generated from the $i$-th agent, and then agent $i$ sends the perturbed primal variable $\tilde{x}_i^{k+1}$ to all the neighboring agents $j \in \mathcal{V}_i$. Finally, each agent updates the dual variable $\alpha_i^{k+1}$ through 
\begin{align}\label{ua}
    \alpha_i^{k+1}=\alpha_i^{k}+\eta \left(|\mathcal{V}_i|\tilde{x}_i^{k+1} - \sum_{j\in \mathcal{V}_i }\tilde{x}_j^{k+1}\right).
\end{align}
During the whole iterative process, we assume that there is a rate $R \in (0,1)$ to describe the decrease of noise variance ${(\sigma^2)}_i^{k+1}$. In other words, we consider a Gaussian mechanism with a linearly decaying variance to provide differential privacy guarantee and achieve a significantly fast convergence rate. 

Before showing P-ADMM satisfies dynamic zCDP, we first estimate the sensitivity of the local primal variable $x_i^{k+1}$ as shown in the following lemma.
\begin{lemma}
Assume the $\|\nabla f_i(x)\|_2 \leq V$, and then the sensitivity of local primal variable $x_i^{k+1}$, denoted by $\Delta_i$, is $\frac{V}{ \eta|\mathcal{V}_i|}$.
\end{lemma}
\begin{proof}
According to subproblem (\ref{u-xx}) and sensitivity definition (\ref{sens}), for two neighboring datasets $D_i$ and $\hat{D}_i$, we have
\begin{align*}
  x_{i,D_i}^{k+1} = -& \frac{1}{2 \eta|\mathcal{V}_i|}\nabla  f_i(x_i^{k+1}  ,D_i) \\ &+\frac{1}{2|\mathcal{V}_i|}\left(|\mathcal{V}_i|x_i^k + \sum_{j\in \mathcal{V}_i }x_j^k\right) - \frac{1}{2 \eta|\mathcal{V}_i|} \alpha_i^k ,
\end{align*}
\begin{align*}
  x_{i,\hat{D}_i}^{k+1} = -& \frac{1}{2 \eta|\mathcal{V}_i|}\nabla   f_i(x_i^{k+1}  ,\hat{D}_i) \\ &+\frac{1}{2|\mathcal{V}_i|}\left(|\mathcal{V}_i|x_i^k + \sum_{j\in \mathcal{V}_i }x_j^k\right) - \frac{1}{2 \eta|\mathcal{V}_i|} \alpha_i^k .
\end{align*}
The sensitivity of $x_i^{k+1}$ is
\begin{align*}
\Delta_i &= \|x_{i,D_i}^{k+1}-x_{i,\hat{D}_i}^{k+1} \|_2 \\
&= \frac{1}{2 \eta|\mathcal{V}_i|}\|\nabla f_i(x_i^{k+1}  ,D_i) - \nabla f_i(x_i^{k+1}  ,\hat{D}_i)\|_2 \\
&\leq \frac{V}{ \eta|\mathcal{V}_i|}.\tag*{\qedhere} 
\end{align*}
\end{proof}
The following theorems show the privacy guarantee of P-ADMM.
\begin{theorem}\label{PPADMM}
The P-ADMM algorithm satisfies the dynamic $\rho_i^{k+1}$-zCDP, where $\rho_i^{k+1} = \frac{V^2}{2\eta^2 |\mathcal{V}_i|^2 (\sigma^2)_i^{k+1}}$.
\end{theorem}
\begin{proof}
The privacy loss variable of $\tilde{x}_i^{k+1}$ at an output $o$ over two neighboring datasets $D_i$ and $\hat{D}_i$ is given by
\begin{align*}
      Z_i^{k+1}(o) &= \ln\frac{{\rm Pr}[\tilde{x}_{i,D_i}^{k+1}= o] }{{\rm Pr}[\tilde{x}_{i,\hat{D}_i}^{k+1}= o ]}.
\end{align*}
Since $\tilde{x}_i^{k+1} = x_i^{k+1} + \xi_i^{k+1} $ and $\xi_i^{k+1}\sim \mathcal{N}(0, {(\sigma^2)}_i^{k+1} I_d)$, the probability distribution $\tilde{x}_{i,D_i}^{k+1}$ is  $\mathcal{N}(x_{i,D_i}^{k+1}, {(\sigma^2)}_i^{k+1} I_d)$, and the probability distribution of $\tilde{x}_{i,\hat{D}_i}^{k+1}$ is $\mathcal{N}(x_{i,\hat{D}_i}^{k+1}, {(\sigma^2)}_i^{k+1} I_d)$.
According to Lemma 2.5 in \cite{bun2016concentrated} and $\forall \tau \in (1,\infty)$, the R\'enyi divergence of this two probability distribution is given by 
\begin{align*}
    &D_{\tau}(\mathcal{N}(x_{i,D_i}^{k+1}, {(\sigma^2)}_i^{k+1} I_d) \|\mathcal{N}(x_{i,\hat{D}_i}^{k+1}, {(\sigma^2)}_i^{k+1} I_d) ) \\
    &= \frac{\tau \|x_{i,D_i}^{k+1} - x_{i,\hat{D}_i}^{k+1} \|_2^2}{{2(\sigma^2)}_i^{k+1}}\\
    & =\frac{\tau \Delta_i^2 }{{2(\sigma^2)}_i^{k+1}}.
\end{align*}
Then,  
\begin{align*}
    &\mathbb{E}[e^{(\tau -1)Z_i^{k+1}(o)}] \\
    &\leq e^{(\tau -1)     D_{\tau}(\mathcal{N}(x_{i,D_i}^{k+1} ,{(\sigma^2)}_i^{k+1} I_d) \|\mathcal{N}(x_{i,\hat{D}_i}^{k+1}, {(\sigma^2)}_i^{k+1} I_d) ) }\\
    & = e^{(\tau -1)\tau \Delta_i^2/[{{2(\sigma^2)}_i^{k+1}}]}\\
    &\leq e^{(\tau -1)\tau\frac{ V^2}{ 2\eta^2|\mathcal{V}_i|^2 {{(\sigma^2)}_i^{k+1}}}} \\
    &= e^{(\tau -1)\tau \rho_i^{k+1}}.
\end{align*}
Therefore, P-ADMM provides the dynamic $\rho_i^{k+1}$-zCDP at each agent $i$ with $\rho_i^{k+1} = \frac{V^2}{2\eta^2 |\mathcal{V}_i|^2 (\sigma^2)_i^{k+1}}$.
\end{proof}

Similar to dynamic differential privacy in~\cite{zhang2016dual}, the limitations of dynamic zero-concentrated differential privacy is that the privacy constraint exists at each iteration and the privacy loss is only calculated at a particular iteration. Since the primal variables are generated and broadcasted during the entire iterative process, then there is a great need to compute the total privacy loss during all iterations. The following theorem therefore qualifies the total privacy loss of P-ADMM by using $(\epsilon,\delta)$-differential privacy.

\begin{theorem}
For any $R\in(0,1)$ and $\delta \in(0,1)$, the P-ADMM algorithm is $(\epsilon,\delta)$-differential privacy with 
\begin{align*}
     \epsilon =\max_{i\in \mathcal{N}}~\left(\frac{\rho_i^1(1-R^K)}{R^{K-1} - R^K} + 2\sqrt{\frac{\rho_i^1(1-R^K) \ln{1/\delta}}{R^{K-1} - R^K} }\right),
\end{align*}
where $ \rho_i^{ 1} = \frac{V^2}{2\eta^2 |\mathcal{V}_i|^2 (\sigma^2)_i^{ 1}}$.
\end{theorem}
\begin{proof}
According to Theorem~\ref{PPADMM}, P-ADMM satisfies dynamic $\rho_i^{k+1}$-zCDP. It ensures that each primal variable $x_i^{k+1}$ perturbed by noise drawn from the Gaussian distribution $\mathcal{N}(0, {(\sigma^2)}_i^{k+1} I_{d})$ is $\rho_i^{k+1}$-zCDP at $k+1$ iteration. 
By the composition theorem in Lemma~\ref{comp} and for each agent, P-ADMM provides $\sum_{k=0}^{K-1}\rho_i^{k+1}$-zCDP. Since we assume the noise variance ${(\sigma^2)}_i^{k+1} = R {(\sigma^2)}_i^{k}$, together with the result in Theorem \ref{PPADMM}, the relationship between $\rho_i^{k+1}$ and $\rho_i^{k}$ is $\rho_i^{k+1} = \rho_i^{k}/R$, and P-ADMM is $\rho_i^{total}$-zCDP for each agent with $\rho_i^{total} = \rho_i^1\frac{1-R^K}{R^{K-1} - R^K}$. By Lemma \ref{cdptodp} and $\forall \delta \in (0,1)$, P-ADMM satisfies ($\epsilon_i^{total},\delta$)-differential privacy, where
\begin{align*}
    \epsilon_i^{total} &= \rho_i^{total} + 2\sqrt{\rho_i^{total} \ln{1/\delta}} \\
    &= \frac{\rho_i^1(1-R^K)}{R^{K-1} - R^K} + 2\sqrt{\frac{\rho_i^1(1-R^K) \ln{1/\delta}}{R^{K-1} - R^K} }.
\end{align*}
Therefore, considering all of agents, the total privacy loss of P-ADMM is bounded by $(\epsilon,\delta)$-differential privacy with 
\begin{align*}
     \epsilon &=\max_{i\in \mathcal{N}}~ \epsilon_i^{total} \\
     &=\max_{i\in \mathcal{N}}~\left(\frac{\rho_i^1(1-R^K)}{R^{K-1} - R^K} + 2\sqrt{\frac{\rho_i^1(1-R^K ) \ln{1/\delta}}{R^{K-1} - R^K} }\right),
\end{align*}
where $ \rho_i^{ 1} = \frac{V^2}{2\eta^2 |\mathcal{V}_i|^2 (\sigma^2)_i^{ 1}}$, $R \in (0,1)$, and $\delta \in (0,1)$.
\end{proof}

\section{Convergence Analysis}\label{convergence analysis}
In this section, we present the convergence analysis of proposed differential private ADMM algorithm for both strongly convex and general convex objective functions. For a clear presentation, we also present the updates of P-ADMM in matrix format. 

At $k+1$ iteration, since each agent uses the computational result of previous iteration $\tilde{x}^k$ to update variable $x^{k+1}$, and then sends to its neighbors after perturbed, the matrix form of P-ADMM algorithm updates are 
\begin{align}
     \nabla f(x^{k+1}) + \alpha^k +2\eta Mx^{k+1} - \eta L_{+}\tilde{x}^k  &=0 \label{u-xxx},\\
    \alpha^{k+1}-\alpha^{k}-\eta L_{-}\tilde{x}^{k+1} &= 0 ,\label{u-aaa}
\end{align}
where $\tilde{x}^k = x^k + \xi^{k}$ and $\xi^{k} \in \mathbb{R}^{Nd}$ is a vector concatenating all noise variables $\xi_i^{k}$.

Before stating our main results, we first define two auxiliary sequences denoted by $r^k$ and $q^k$. Given the perturbed primal variable $\tilde{x}^{k}$, we let 
\begin{align*}
    r^k = \sum_{s = 0}^{k}Q\tilde{x}^{s}, ~q^k =  \begin{pmatrix}
    r^k \\ x^k
    \end{pmatrix},
    ~G = 
    \begin{pmatrix}
    \eta I &0 \\ 0 &\eta L_+ /2
    \end{pmatrix},
\end{align*}
where $Q= \sqrt{L_-/2}$.
Remember $L_-$ is the Laplcaian matrix and the network is connected, the following properties holds: $L_-$ is positive semi-definite and $Null(L_-) = span\{1\}$. As a result, $Null(Q) = span\{1\}$. We also denote $\phi_{max}(\hat{G})$ and $\phi_{min}(\hat{G})$ as the nonzero largest and smallest singular value for a semidefinite matrix $\hat{G}$, respectively.

Using the second update (\ref{u-aaa}) of the matrix form of P-ADMM, the first step (\ref{u-xxx}) can be written as
\begin{align}\label{le1}
    x^{k+1} = -\frac{M^{-1} \nabla f(x^{k+1})}{2\eta} + \frac{M^{-1}L_+ \tilde{x}^k}{2}  - \frac{M^{-1 }L_-}{2} \sum_{s=0}^k \tilde{x}^s.
\end{align}
Based on the auxiliary sequence $r^k$ and the fact $M = (L_-+L_+)/2$, the above equation (\ref{le1}) can be rewritten as
\begin{align}\label{lem2}
    \frac{\nabla f(x^{k+1})}{\eta} +2Qr^{k+1} + L_+(\tilde{x}^{k+1} -\tilde{x}^{k } ) = 2M^{-1} \xi^{k+1}.
\end{align}

\subsection{Strongly Convex Case}
In this case, we consider the objective function $f(x)$ is $\mu$-strongly convex and $v$-smooth. Then, according to the definition of smoothness and strong convexity, we have
\begin{align*}
    f(x^{k+1}) \geq f(x^*) + \left< \nabla f(x^*), x^{k+1}-x^*\right> + \mu\|x^{k+1} - x^*\|_2^2
\end{align*}
and 
\begin{align*}
      \|\nabla f(x^{k+1}) -\nabla f(x^{*})\|_2 \leq v\| x^{k+1}-x^*\|_2 ,
\end{align*}
where we denote $x^*$ as an optimal solution of (\ref{eq:nc-mat}).

Before showing the convergence of P-ADMM, we have the following lemma.
\begin{lemma}[\cite{lirobust}]\label{delta-coverge}
There exists $\hat{q}^* = \begin{pmatrix}  r^* \\ x^* \end{pmatrix}$ for the $r^*$ found in Lemma 4 of \cite{7870649} such that 
\begin{align*}
    &\|q^{k+1} - \hat{q}^* \|_G^2 \\
    & \leq \frac{ \|q^{k } - \hat{q}^* \|_G^2}{1+\zeta} + \frac{P}{1+\zeta} \|\xi^{k+1}\|_2^2 + \frac{1}{1+\zeta} \bigl< \xi^{k+1},  \\
    &~~~\eta L_+ (\tilde{x}^{k+1} - \tilde{x}^k) + 2\eta Q(r^{k+1}- r^*)+ 2\eta M (x^{k+1} - x^*) \bigl>
\end{align*}
with $\zeta = \frac{2(\kappa_2 -1)\mu \phi_{min}^2(Q)\phi_{min}^2(L_+)}{\kappa_2 v^2 \phi_{min}^2(L_+) + 2\mu\phi_{max}^2(L_+)}$, $P=\frac{\zeta \eta^2\kappa_2 \phi_{max}^2(M)}{\phi_{min}^2(Q)} + \frac{\zeta \eta^2 \phi_{max}^2(L_+)\kappa_3}{4}$ and $\eta = \sqrt{\frac{\kappa_1 \kappa_2 (\kappa_3 -1)v^2}{\kappa_3 (\kappa_2 -1)\phi_{min}^2(Q)\phi_{min}^2(L_+)}}$, where $\kappa_1 = 1+\frac{2\mu \phi_{max}^2(L_+)}{v^2\phi_{min}^2(L_+)}$, $\kappa_2 >1$ and $\kappa_3 > 1$ and $\|x\|_G^2$ is denoted as $\|x\|_G^2 = \left< x, Gx\right>$ for any matrix $x$ and $G$.
\end{lemma}

Following Lemma \ref{delta-coverge}, we establish the linear convergence of P-ADMM for the strongly convex problem in the following theorem. 
\begin{theorem}\label{strongconvex}
Suppose the objective function $f(x)$ is $\mu$-strongly convex and $v$-smooth. In P-ADMM, let $S_2 = \frac{4}{(1+4\theta)\phi_{max}^2(L_+)}$,$W = \frac{(1+4\theta)\phi_{max}^2(L_+)}{(1-b)(1+\zeta-4\theta)\phi_{min}^2(L_+)}$, and $S_3  = \frac{\frac{4\zeta \kappa_2 \phi_{max}^2(M)}{\phi_{min}^2(Q)} +\phi_{max}^2\bigl( \sqrt{\zeta} + \sqrt{\frac{2(\kappa_2-1)\phi_{min}^2(Q)}{\theta \kappa_1 \kappa_2}}\bigl)^2}{(1-b)(1+\zeta)(1+\zeta-4\theta)\phi_{min}^2(L_+)} + \frac{b(\kappa_4 -1)}{1-b}$. If one chooses $b$ and $\kappa_2$ satisfying $(1-b)(1+\zeta)\phi_{min}^2(L_+)> \phi_{max}^2(L_+)$ and chooses $0 < \theta < \min \{ H_1, H_2 \}$ with $H_1 = \frac{b(1+ \zeta) \phi_{min}^2(L_+)(1- 1/\kappa_4)}{4b\phi_{min}^2(L_+)(1- 1/\kappa_4) + 16\phi_{max}^2(M) }$, $H_2 = \frac{(1-b)(1+\zeta)\phi_{min}^2(L_+) - \phi_{max}^2(L_+)}{4\phi_{max}^2(L_+) + 4(1-b)\phi_{min}^2(L_+)}$,
then the following holds for 
\begin{align*}
      \mathbb{E}[\|\tilde{x}^{k+1} - x^*\|_2^2] & \leq W^{k+1}\bigl(    \|\tilde{x}^{0 } 
      - x^*\|_2^2 + S_2 \|r^{0 } - r^*\|_2^2 \nonumber   \\
      &~~~+ \frac{dS_3 \sum_{i=1}^N (\sigma^2)_i^1 W}{W-R} \bigl)
\end{align*}
with $0<R<W<1$, where the expectation is taking with respect to the noise,   $\kappa_1$ and $\zeta$ are found in Lemma~\ref{delta-coverge}, $\kappa_3 = \sqrt{\frac{v^2\phi_{min}^2(L_+) + 2\mu\phi_{max}^2(L_+)}{\theta \kappa_1 v^2\phi_{min}^2(L_+)}} +1$ and $\kappa_4 > 1$.
\end{theorem}

\begin{proof}
The result in Lemma~\ref{lemma6} can be rewritten as 
\begin{align}
   & \|\tilde{x}^{k+1} - x^*\|_2^2 + S_1 \|r^{k+1} - r^*\|_2^2\nonumber \\
    &\leq W(    \|\tilde{x}^{k } - x^*\|_2^2 + S_2 \|r^{k } - r^*\|_2^2 ) + S_3\|\xi^{k+1}\|_2^2 ,\nonumber
\end{align}
where $S_1 = \frac{4}{(1-b)\phi_{min}^2(L_+)}$, $S_2 = \frac{4}{(1+4\theta)\phi_{max}^2(L_+)}$,$W = \frac{(1+4\theta)\phi_{max}^2(L_+)}{(1-b)(1+\zeta-4\theta)\phi_{min}^2(L_+)}$, and 
\begin{align*}
    S_3 &= \frac{\frac{4\zeta \kappa_2 \phi_{max}^2(M)}{\phi_{min}^2(Q)} +\phi_{max}^2\bigl( \sqrt{\zeta} + \sqrt{\frac{2(\kappa_2-1)\phi_{min}^2(Q)}{\theta \kappa_1 \kappa_2}}\bigl)^2}{(1-b)(1+\zeta)(1+\zeta-4\theta)\phi_{min}^2(L_+)} \\
    &~~~+ \frac{b(\kappa_4 -1)}{1-b}
\end{align*} 
by setting $\kappa_3 = \sqrt{\frac{v^2\phi_{min}^2(L_+) + 2\mu\phi_{max}^2(L_+)}{\theta \kappa_1 v^2\phi_{min}^2(L_+)}} +1$.
Furthermore, by setting $(1-b)\phi_{min}^2(L_+) \leq (1+\theta)\phi_{max}^2(L_+)$, we have 
$S_1 \geq S_2$.
Therefore, we obtain 
\begin{align*}
     & \|\tilde{x}^{k+1} - x^*\|_2^2 + S_1 \|r^{k+1} - r^*\|_2^2 \leq W^{k+1}\bigl(    \|\tilde{x}^{0 } 
      - x^*\|_2^2 \\
      &~~~+ S_2 \|r^{0 } - r^*\|_2^2  + \sum_{s  =1}^{k+1}W^{-s}S_3\|\xi^{k+1}\|_2^2\bigl).
\end{align*}
Since $ S_1 \|r^{k+1} - r^*\|_2^2>0$, we have
\begin{align}\label{st_cover1}
      \|\tilde{x}^{k+1} - x^*\|_2^2 & \leq W^{k+1}\bigl(    \|\tilde{x}^{0 } 
      - x^*\|_2^2 + S_2 \|r^{0 } - r^*\|_2^2 \nonumber   \\
      &~~~+ \sum_{s  =1}^{k+1}W^{-s}S_3\|\xi^{k+1}\|_2^2\bigl).
\end{align}
Next, we constrain $0< W <1$ by choosing $\theta$ from
\begin{align*} 
   0< \theta <\frac{(1-b)(1+\zeta)\phi_{min}^2(L_+) - \phi_{max}^2(L_+)}{4\phi_{max}^2(L_+) + 4(1-b)\phi_{min}^2(L_+)}.
\end{align*}
Together with the constraint of $\theta$ in Lemma~\ref{lemma6}, $\theta$ should be chosen as $0 < \theta < \min \{H_1, H_2 \}$, where 
\begin{align*}
    H_1 = \frac{b(1+ \zeta) \phi_{min}^2(L_+)(1- 1/\kappa_4)}{4b\phi_{min}^2(L_+)(1- 1/\kappa_4) + 16\phi_{max}^2(M) }
\end{align*}
and 
\begin{align*}
    H_2 = \frac{(1-b)(1+\zeta)\phi_{min}^2(L_+) - \phi_{max}^2(L_+)}{4\phi_{max}^2(L_+) + 4(1-b)\phi_{min}^2(L_+)}.
\end{align*}
Taking expectation of (\ref{st_cover1}) and choosing variance decrease rate $R$ less than the linear convergence rate $W$, we obtain
\begin{align*} 
      \mathbb{E}[\|\tilde{x}^{k+1} - x^*\|_2^2] & \leq W^{k+1}\bigl(    \|\tilde{x}^{0 } 
      - x^*\|_2^2 + S_2 \|r^{0 } - r^*\|_2^2 \nonumber   \\
      &~~~+ \sum_{s  =1}^{k+1}W^{-s}S_3 \mathbb{E}\|\xi^{k+1}\|_2^2\bigl) \nonumber\\
      & \leq W^{k+1}\bigl(    \|\tilde{x}^{0 } 
      - x^*\|_2^2 + S_2 \|r^{0 } - r^*\|_2^2 \nonumber   \\
      &~~~+ \frac{dS_3 \sum_{i=1}^N (\sigma^2)_i^1 W}{W-R} \bigl).
\end{align*}
In the first inequality, we use $\mathbb{E}[\xi^{k+1}] = 0$. In the second inequality, we use the sum of infinity terms of the geometric sequence.
\end{proof}
\begin{remark}
Under a strong convexity assumption, P-ADMM exhibits the same linear convergence result of non-private decentralized ADMM in \cite{shiwadmm}. As we can see from above theorem, if the variance of Gaussian noise decays much faster than the distance $\|\tilde{x}^{k+1} - x^*\|_2^2$, then P-ADMM will converge at a linear rate to the minimizer.
\end{remark}
\subsection{Convex Case}
Some objectives like Lasso and group Lasso are non-strongly convex, so, in this case, we consider the objective function $f(x)$ is general convex. In other words, by the definition of convex, we have 
\begin{align*}
    f(x^{k+1}) -f(x^*) \leq \left< x^{k+1}-x^*, \nabla f(x^{k+1})\right>,
\end{align*}
where $x^*$ is the optimal solution of (\ref{eq:nc-mat}).

\begin{lemma} \label{pro1}
For any $r \in \mathbb{R}^{Nd}$ and $k>0$, we have
\begin{align*}
   & \frac{f(x^{k+1}) - f(x^*)}{\eta} + \left<2r,Qx^{k+1} \right>  \\ 
& \leq \frac{1}{\eta} \left(  \|q^k -q^*\|_G^2 - \| q^{k+1} -q^*\|_G^2  \right) + \frac{\phi_{max}^2(L_+)}{2\phi_{min}(L_-)} \|\xi^k  \|_2^2 \\
    &~~~+\left<\xi^{k+1}, 2Q(r^{k+1} - r) \right>,
\end{align*}
where $q^* = \begin{pmatrix}
    r \\ x^*
    \end{pmatrix}$, and $\|x\|_G^2$ is denoted as $\|x\|_G^2 = \left< x, Gx\right>$ for any matrix $x$ and $G$. 
\end{lemma}
\begin{proof}
Using convexity and (\ref{lem2}), $\forall r \in \mathbb{R}^{Nd}$, we have
\begin{align}\label{prpos_1}
    &\frac{f(x^{k+1}) - f(x^*)}{\eta} + \left<2r,Qx^{k+1}\right> \nonumber\\
    & \leq \frac{1}{\eta}\left< x^{k+1}-x^*, \nabla f(x^{k+1})\right> + \left<2r,Qx^{k+1}\right> \nonumber\\
    & \leq \bigl< x^{k+1}-x^*, 2M^{-1} \xi^{k+1} -2Q(r^{k+1}-r) \nonumber\\
    &~~~- L_+(\tilde{x}^{k+1} -\tilde{x}^{k } )  \bigl> \nonumber\\
    & = \left< x^{k+1} - x^*, L_+(x^k - x^{k+1})\right> +\bigl< x^{k+1} - x^*, \nonumber\\ &~~~~L_+(\tilde{x}^k - x^{k})\bigl>  
     +\left< x^{k+1} - x^*, 2Q(r -r^{k+1} )\right> \nonumber\\
    &~~~+ \left< x^{k+1} - x^*, L_-(\tilde{x}^{k+1} - x^{k+1} )\right> \nonumber\\
    & = \left< x^{k+1} - x^*, L_+(x^k - x^{k+1})\right> +\bigl< x^{k+1} - x^*, \nonumber\\ &~~~~L_+(\tilde{x}^k - x^{k})\bigl> +\left< \tilde{x}^{k+1} - x^*, 2Q(r -r^{k+1} )\right>\nonumber \\
    &~~~+ \left< x^{k+1} - x^*, L_-(\tilde{x}^{k+1} - x^{k+1} )\right>  \nonumber\\
    &~~~+\left<\xi^{k+1},2Q(r -r^{k+1} \right>.
\end{align}

Recall $Null(Q) = span\{1\}$ and $x^*$ is the optimal consensus solution, so we have $\left< x^*, Q\right> = 0$. By the definition of sequence $r^k$, we have $\tilde{x}^{k+1} = r^{k+1} - r^k$. Thus, (\ref{prpos_1}) is simplified as
\begin{align*}
    &\frac{f(x^{k+1}) - f(x^*)}{\eta} + \left<2r,Qx^{k+1} \right>  \\
    & \leq \frac{1}{\eta} \left(  \| q^k -q^*\|_G^2 - \| q^{k+1} -q^*\|_G^2  - \| q^{k+1} -q^k\|_G^2 \right) \\
    &~~~+  \left< x^{k+1} - x^* , L_+(\tilde{x}^{k} - x^k)  \right> +\left<\xi^{k+1}, 2Q(r^{k+1} - r) \right> \\
    &~~~+ \left< x^{k+1} - x^* , L_-(\tilde{x}^{k+1} - x^{k+1})  \right> \\
    & = \frac{1}{\eta} \left(  \| q^k -q^*\|_G^2 - \| q^{k+1} -q^*\|_G^2  \right)- \|Qx^{k+1}\|_2^2 \\ &~~~-\|Q \xi^{k+1}\|_2^2+2 \left< \frac{L_+}{2}(x^{k+1} -x^* ), \tilde{x}^k - x^k\right>  \\
    &~~~ +\left<\xi^{k+1}, 2Q(r^{k+1} - r) \right> \\
    & \leq \frac{1}{\eta} \left(  \| q^k -q^*\|_G^2 - \| q^{k+1} -q^*\|_G^2  \right) \\
    &~~~ - \frac{\phi_{min} (L_-)}{2}\| x^{k+1} - x^{*}\|_2^2 -\|Q \xi^{k+1}\|_2^2  \\
    &~~~+  \frac{2\phi_{min}(L_-)}{\phi_{max}^2(L_+)}\|\frac{L_+}{2}(x^{k+1} -x^* )\|_2^2 \\
    &~~~+ \frac{\phi_{max}^2(L_+)}{2\phi_{min}(L_-)} \|\tilde{x}^k - x^k \|_2^2 +\left<\xi^{k+1}, 2Q(r^{k+1} - r) \right> \\
    & \leq \frac{1}{\eta} \left(  \| q^k -q^*\|_G^2 - \| q^{k+1} -q^*\|_G^2  \right) + \frac{\phi_{max}^2(L_+)}{2\phi_{min}(L_-)} \|\xi^k  \|_2^2 \\
    &~~~+\left<\xi^{k+1}, 2Q(r^{k+1} - r) \right>.
\end{align*}
In the second inequality, we use $\frac{1}{t}\hat{a}^2 + t\hat{b}^2 \geq 2\hat{a} \hat{b}$ for $t>0$.
\end{proof}

Based on Lemma \ref{pro1}, we give the following theorem to establish the convergence rate of P-ADMM for the general convex problem.
\begin{theorem}\label{generalconvex}
Suppose the objective function $f(x)$ is general convex. In P-ADMM, for any $K>0$, we have
\begin{align*}
    &\mathbb{E}[f(\hat{x}^K) - f(x^*)] \leq \frac{\eta\|Qx^0\|_2^2}{K} \nonumber\\
    &+ \frac{\eta \|x^0 - x^*\|_{\frac{L_+ }{2}}^2}{K} + \frac{1}{K}\frac{\eta d \phi_{max}^2(L_+) \sum_{i=1}^N (\sigma^2)_i^1}{2\phi_{min}(L_-)(1-R)}
\end{align*}
with $0<R<1$, where the expectation is taking with respect to the noise and $\hat{x}^K  = \frac{1}{K} \sum_{k=0}^{K-1} x^{k+1} $.
\end{theorem}
\begin{proof}

\begin{figure*}[t]
	\centering   
	\includegraphics[width=0.265\textwidth]{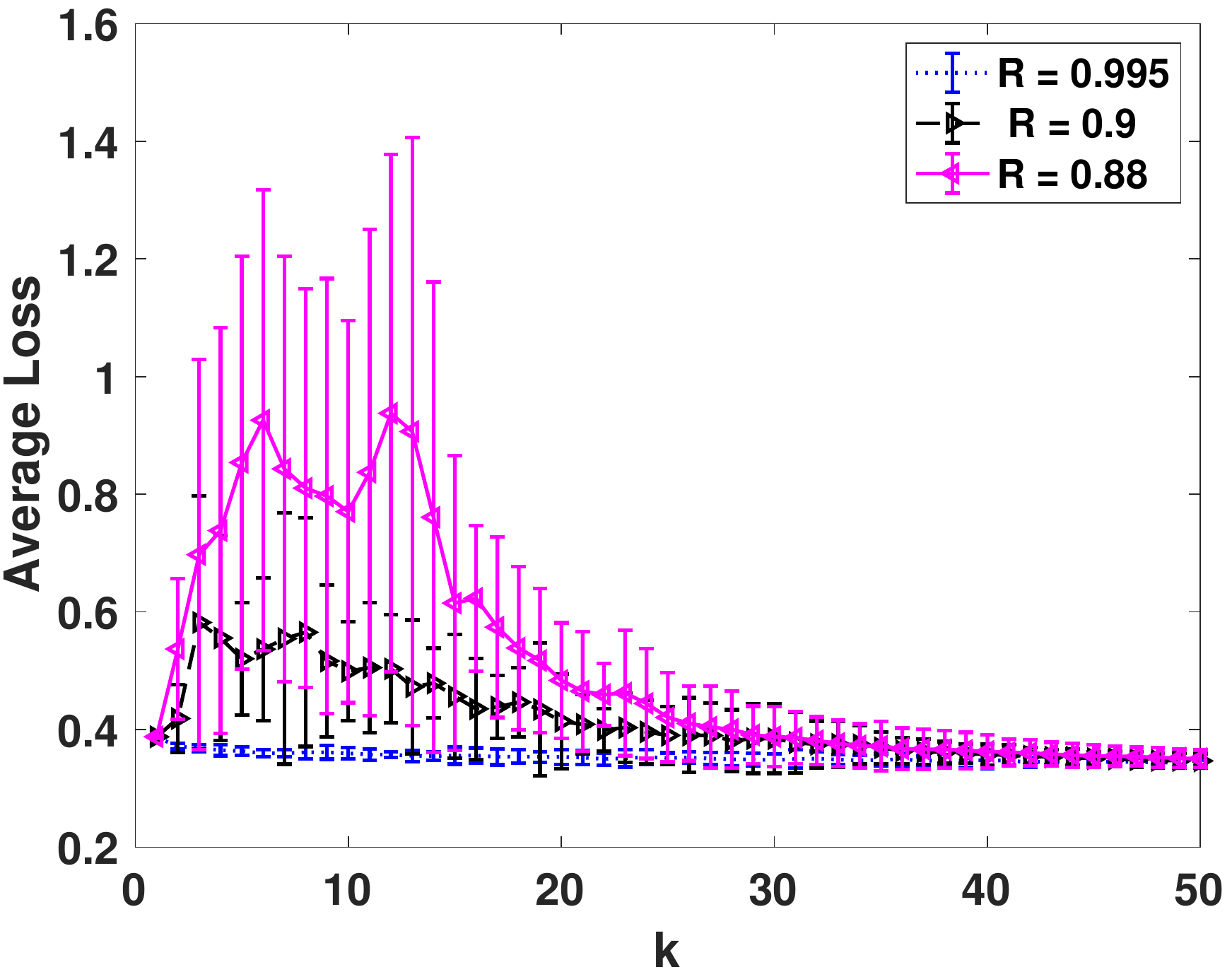}
	\includegraphics[width=0.265\textwidth]{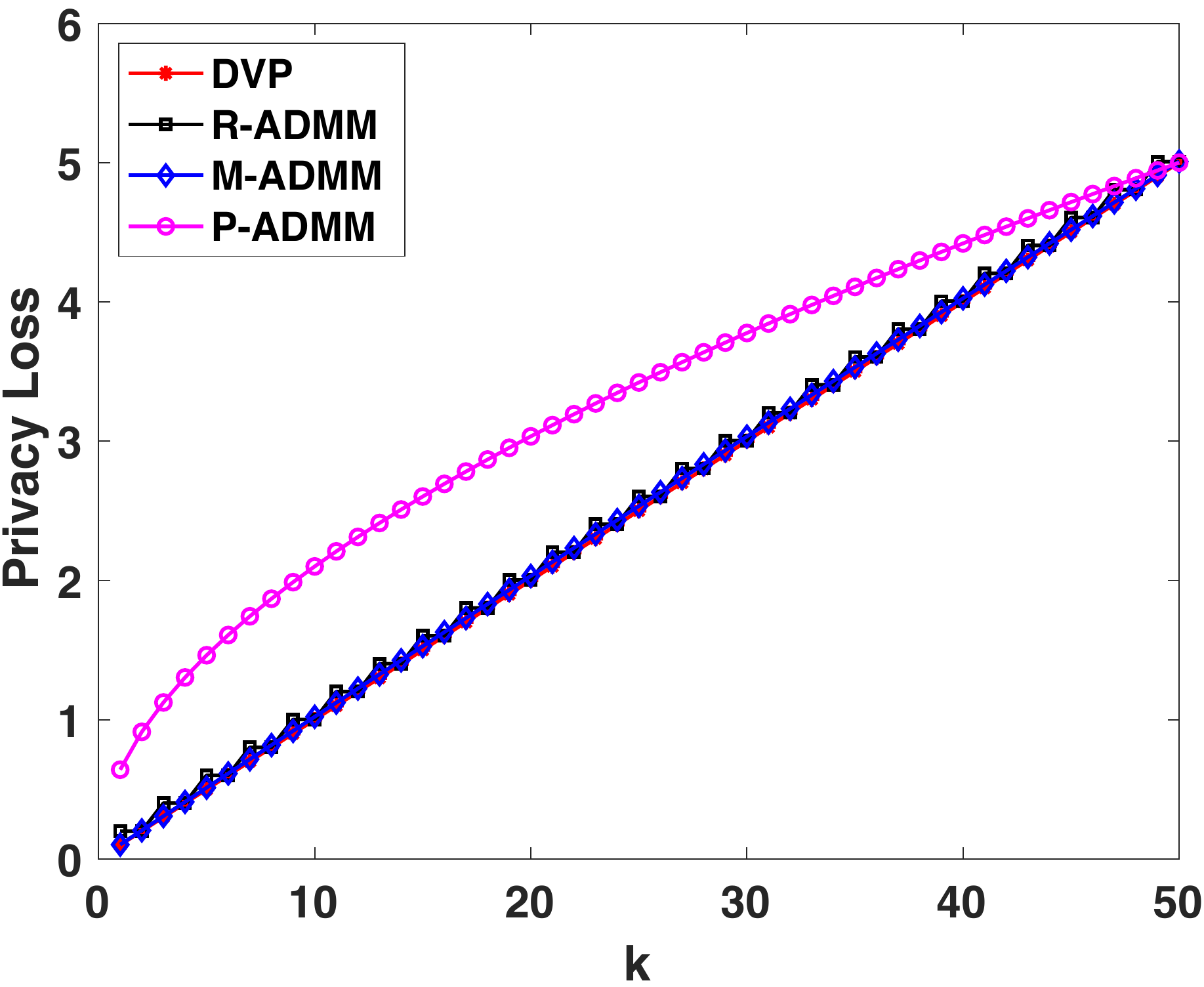}
	\includegraphics[width=0.45\textwidth]{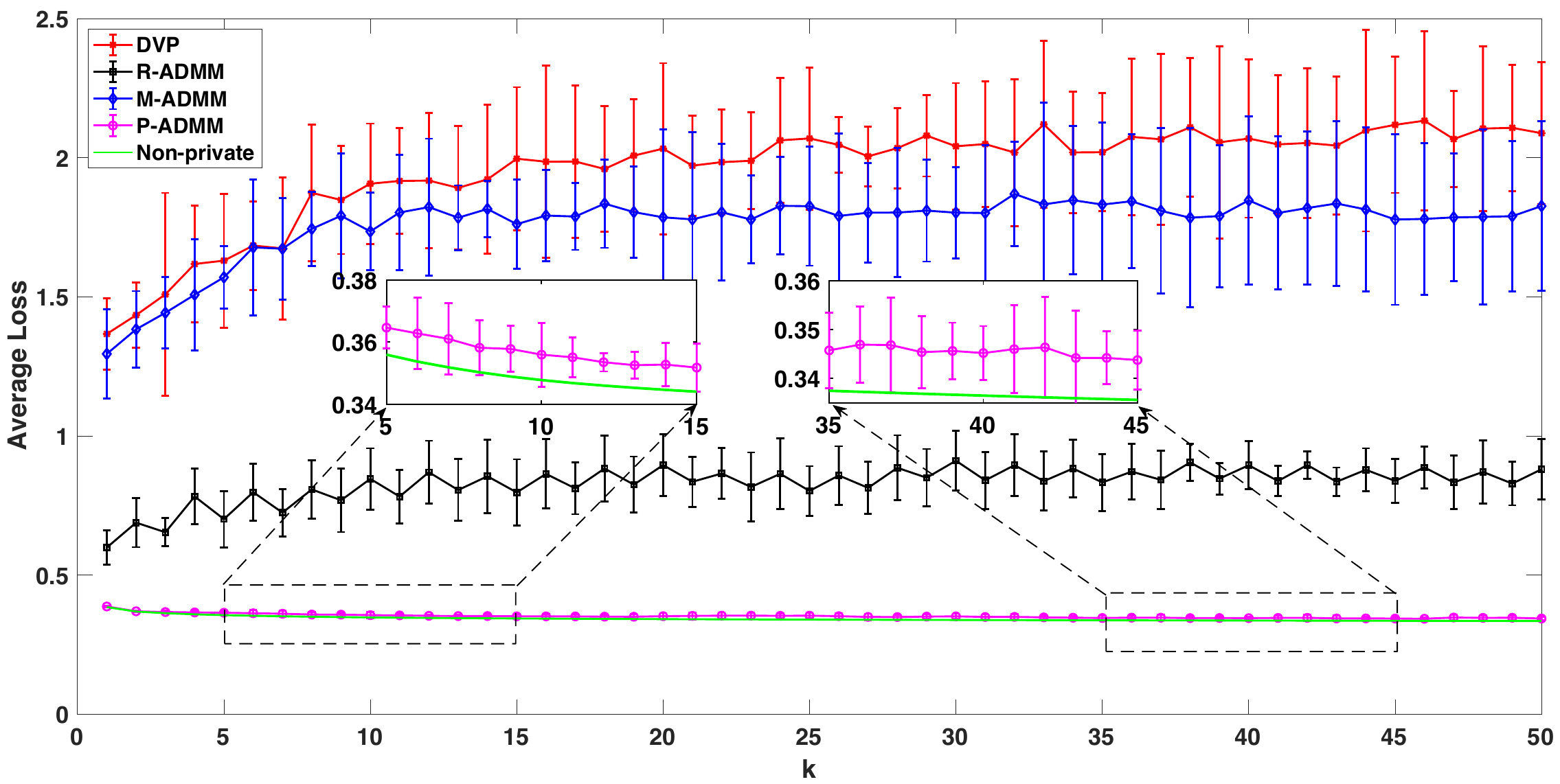}

	\caption{Compare convergence and privacy for $L_2$-regularized logistic regression: Total privacy loss $\epsilon = 5$.}
	\label{fig_rg_5}
\end{figure*}

\begin{figure*}[t]
	\centering   
	\includegraphics[width=0.265\textwidth]{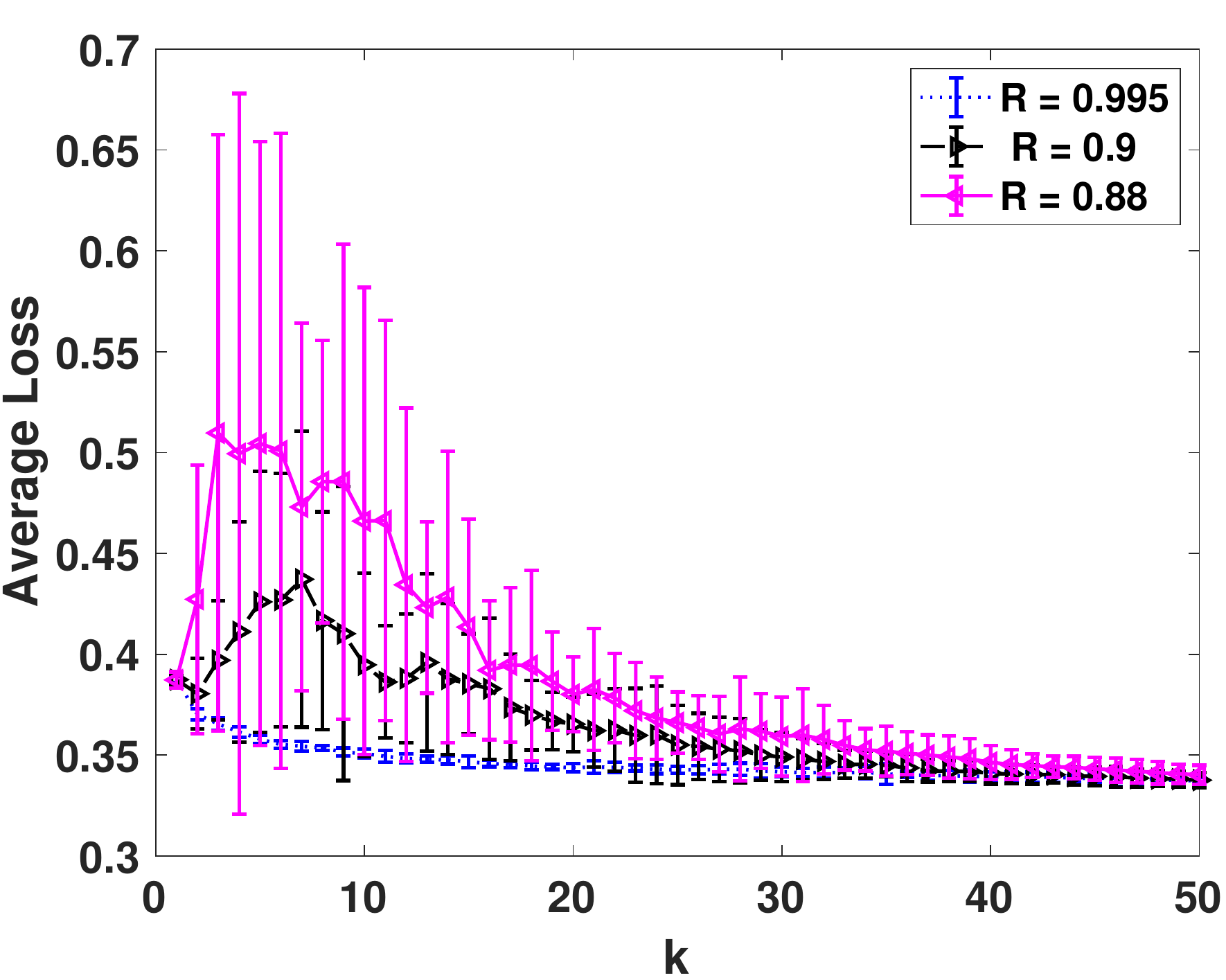}
	\includegraphics[width=0.265\textwidth]{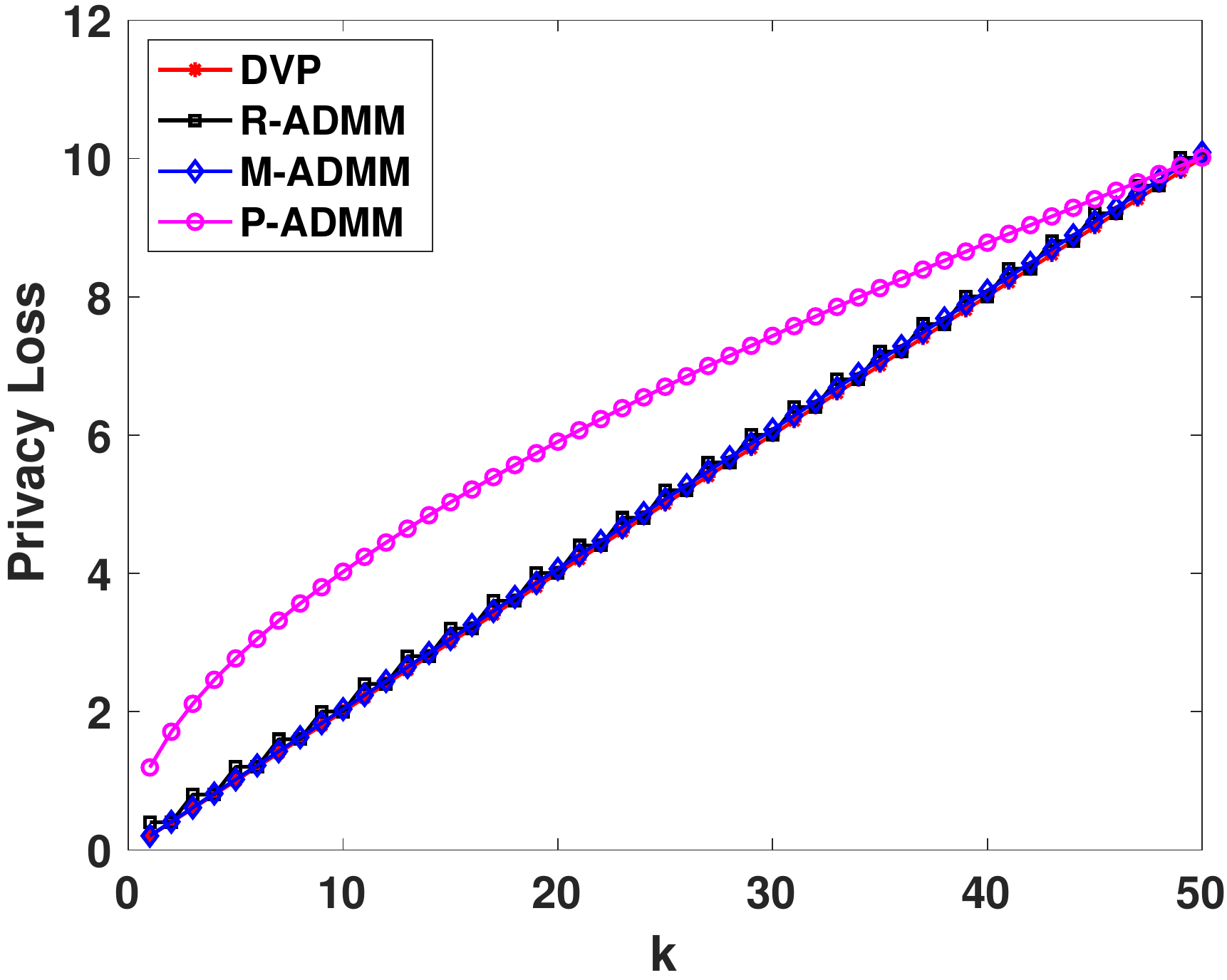}
	\includegraphics[width=0.45\textwidth]{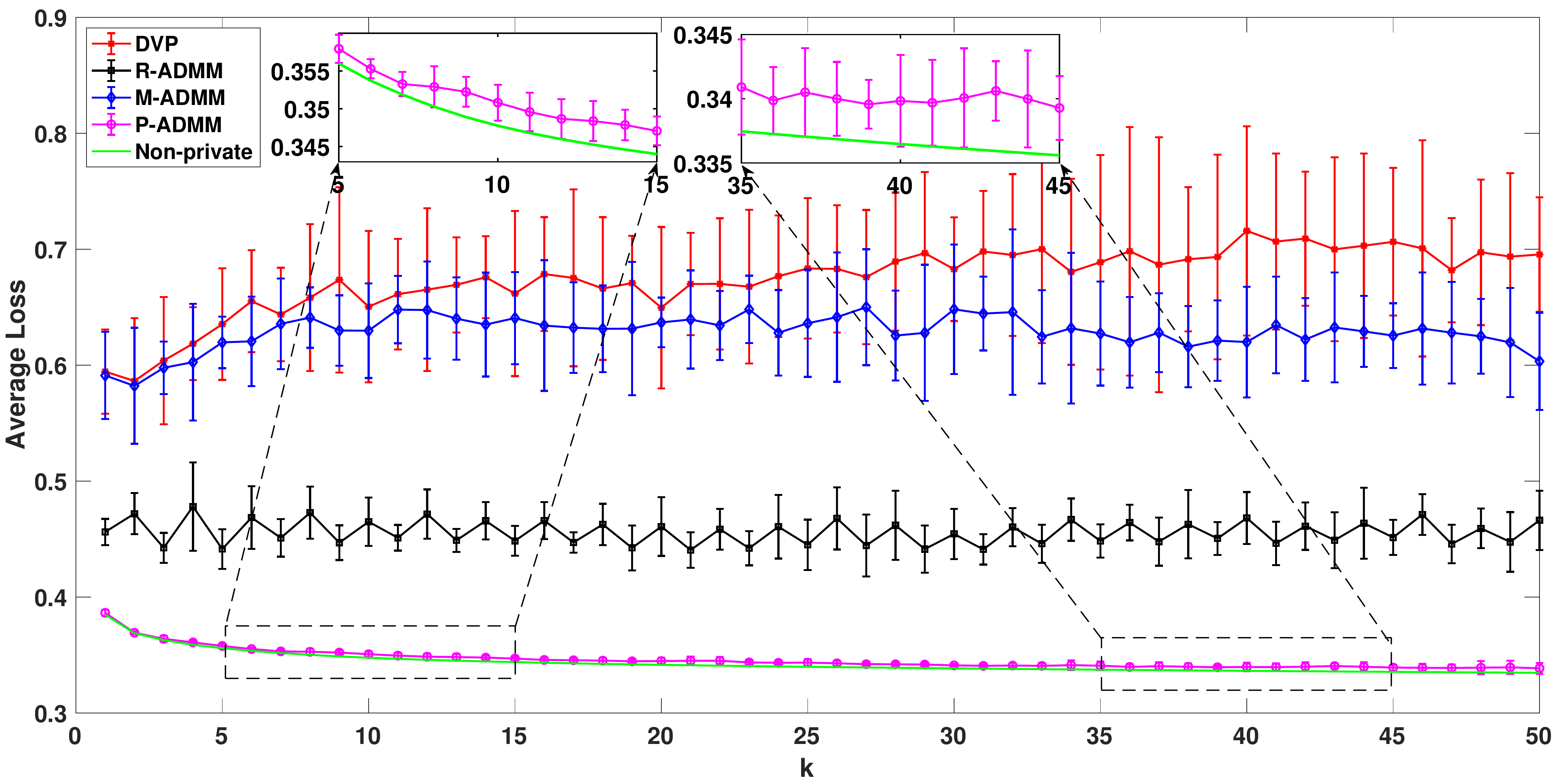}
	\caption{Compare convergence and privacy for $L_2$-regularized logistic regression: Total privacy loss $\epsilon = 10$.}
	\label{fig_rg_10}
\end{figure*}
\begin{figure*}[t]
	\centering   
	\subfigure[Total privacy loss $\epsilon = 5$ .]{\includegraphics[width=0.49\textwidth]{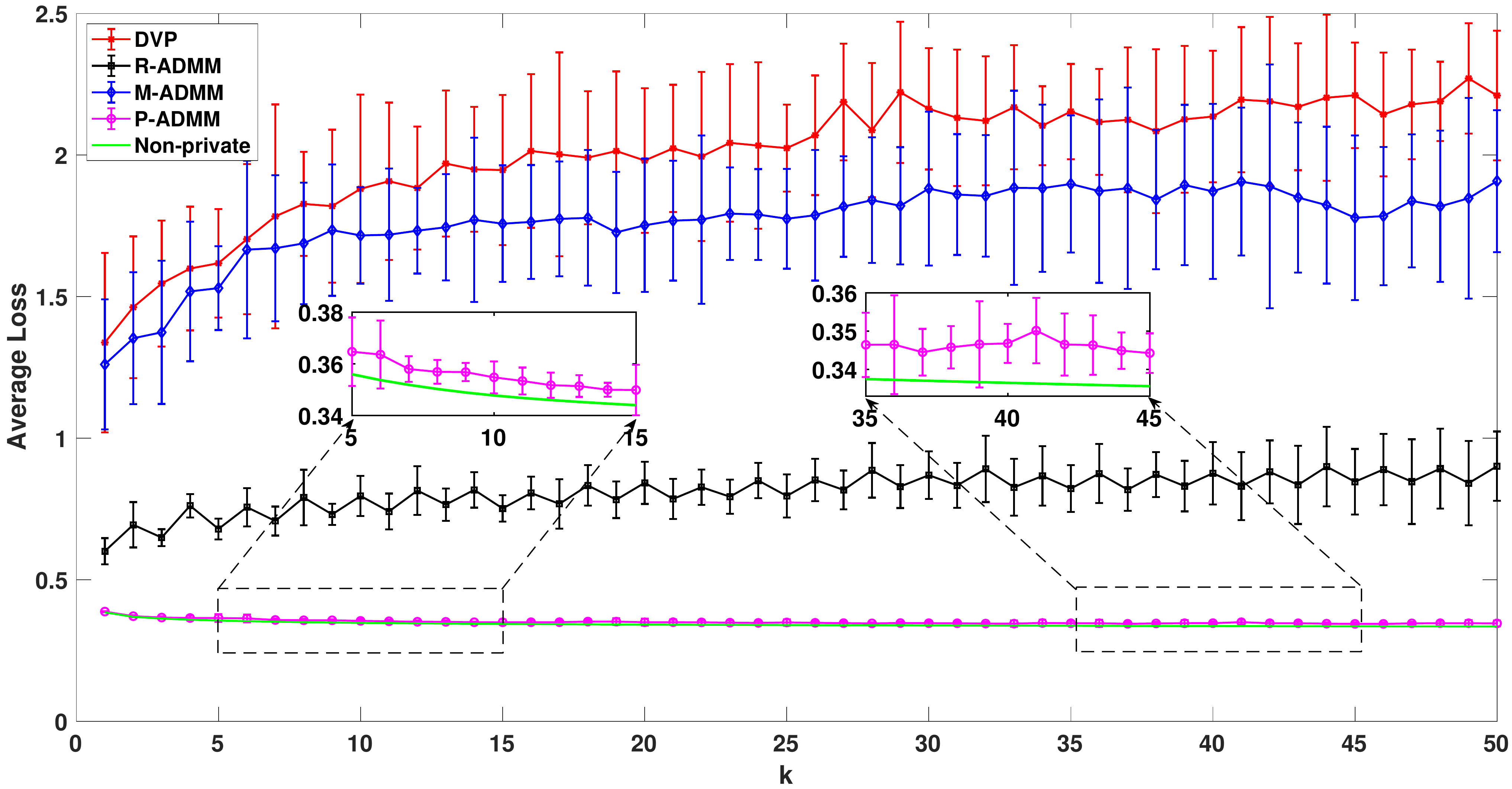}}
	\subfigure[Total privacy loss $\epsilon = 10$ .]{\includegraphics[width=0.49\textwidth]{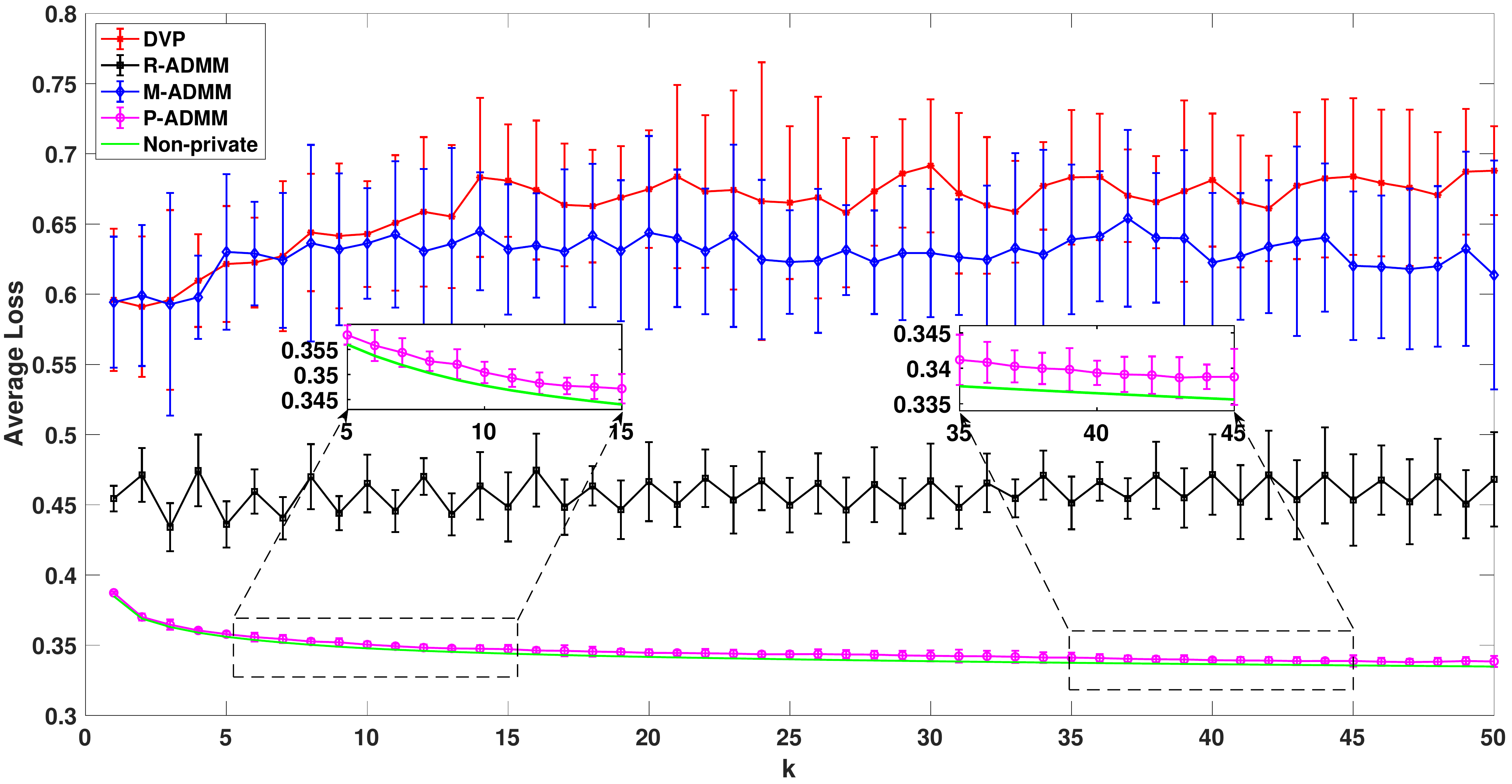}}
	\caption{Compare convergence and privacy for unregularized logistic regression.}
	\label{fig3.11}
\end{figure*}

Summing Lemma~\ref{pro1} from $k =0$ to $k = K-1$, we have
\begin{align*}
    &\frac{1}{\eta} (\sum_{k = 0}^{K-1}f(x^{k+1}) - f(x^*)) +\left<2r, \sum_{k =0}^{K-1}Qx^{k+1}\right> \\
    &\leq \sum_{k =0}^{K-1}\left( \frac{\phi_{max}^2(L_+)}{2\phi_{min}(L_-)} \|\xi^k  \|_2^2 + \left<\xi^{k+1}, 2Q(r^{k+1} - r) \right>\right) \\
    & ~~~   + \frac{1}{\eta} \| q^0 -q^*\|_G^2.
\end{align*}
By setting $r = 0$, there is
\begin{align*}
    &\frac{1}{\eta} (\sum_{k = 0}^{K-1}f(x^{k+1}) - f(x^*))\\
    &\leq \sum_{k =0}^{K-1}\left( \frac{\phi_{max}^2(L_+)}{2\phi_{min}(L_-)} \|\xi^k  \|_2^2 + \left<\xi^{k+1}, 2Q r^{k+1}   \right>\right) \\
    & ~~~   +  \|Qx^0\|_2^2 +\| x^0 -x^*\|_{\frac{L_-}{2} }^2.
\end{align*}
By taking expectation of above function and using Jensen's inequality, we have
\begin{align*}
    &\mathbb{E}[f(\hat{x}^K) - f(x^*)]\\
&    \leq \frac{\eta\|Qx^0\|_2^2}{K} + \frac{\eta \|x^0 - x^*\|_{\frac{L_+ }{2}}^2}{K}\\ &~~~~+\frac{\eta}{K}\sum_{k =0}^{K-1}  \frac{\phi_{max}^2(L_+)}{2\phi_{min}(L_-)}\mathbb{E} \|\xi^k  \|_2^2 \\
  &\leq \frac{\eta\|Qx^0\|_2^2}{K} + \frac{\eta \|x^0 - x^*\|_{\frac{L_+ }{2}}^2}{K}\\ &~~~~+\frac{1}{K}\frac{\eta d \phi_{max}^2(L_+) \sum_{i=1}^N (\sigma^2)_i^1}{2\phi_{min}(L_-)(1-R)},
\end{align*}
where $\hat{x}^K  = \frac{1}{K} \sum_{k=0}^{K-1} x^{k+1} $. In the first inequality, we use $\mathbb{E}[\xi^{k+1}] = 0$. In the second inequality, we use the sum of infinity terms of geometric sequence.
\end{proof}
\begin{remark}
Compared with the non-private decentralized ADMM in \cite{7870649}, P-ADMM also achieves a $\mathcal{O}(1/K)$ rate for a general convex problem. As we can see from above theorem, if the variance of Gaussian noise decaying linearly with a rate $R \in (0,1)$, then the averaged function value approaches the minimum function value with a convergence rate of $\mathcal{O}(1/K)$. 
\end{remark}

\section{Numerical Experiments}\label{experiments}
We devote this section to experimentally evaluate the performance of P-ADMM. In particular, we consider the loss function for a commonly used model: logistic regression. 

Given a sample $(y,z)$, where $y$ is a feature vector and the corresponding label is $z \in \{+1,-1\}$, the $L_2$-regularized logistic regression loss function is defined as
\begin{align*}
        \mathcal{L}(z, y, w)  = \log(1+ \exp(-z w^{T} y)) + \Lambda\|w\|_2^2,
\end{align*} 
which is $\Lambda$-strongly convex.

We also use the same real-world dataset as  \cite{zhang2016dual, zhang2018improving,zhang2018recycled}, i.e., the Adult dataset from UCI Machine Learning Repository~\cite{Dua-2017}, which consists of 48,842 personal records, including age, education, occupation, work-class, sex, race, income, etc. The goal is to predict whether the annual income is more than \$50k or not. 

We use the following steps to preprocess the datasets. We first remove all individuals with missing values and then use one-hot encoding method to convert every categorical attribute into a set of binary vectors. After that, we normalize all numerical attributes such that the range of each value is $[0,1]$. Finally, we transform labels of Adult $\{>50 \text{k}, \leq 50\text{k}\}$ to $\{+1, -1\}$. 

In our experiments, we compare our P-ADMM algorithm with benchmark algorithms under both the general convex and strongly convex cases by using unregularzied (i.e. $\Lambda = 0$) and regularized (i.e. $\Lambda \neq 0$) logistic regression loss functions. The benchmark algorithms consist of DVP, M-ADMM and R-ADMM. DVP is the name of dual variable perturbation method presented in \cite{zhang2016dual} that perturbs the dual variable during the iterative process. M-ADMM \cite{zhang2018improving} is based on a penalty perturbation method by adding a perturbation correlated to the step size to ensure differential privacy of ADMM algorithm. R-ADMM is a new variation of the differentially private ADMM algorithm proposed in~\cite{zhang2018recycled} that repeatedly uses the existing computational results to make updates, which improves the privacy-utility tradeoff significantly. To determine the practical cost of privacy for convex optimization, we further use non-private ADMM algorithm \cite{shiwadmm} as a non-private baseline to do classification task on Adult dataset. Moreover, We compare the convergence of P-ADMM with benchmark algorithms under the same total privacy loss. To obtain the same total privacy loss during the whole period, we carefully choose noise parameters, respectively.

For each parameter setting, we conduct 10 independent runs of algorithms. For each time, both the mean and standard deviation of the average loss are recorded. Specially, we define $f_k$ as the average loss over the training dataset of the $k$-th experiment $(1 \leq k \leq 10)$. The mean of the average loss is $f_{mean} = \frac{1}{10}\sum_{k =1}^{10} f_k$ and the standard error is $f_{std} = \sqrt{\frac{1}{10}\sum_{k =1}^{10} (f_k-f_{mean})^2}.$ The smaller the standard deviation $f_{std}$ the more stable of the algorithm. In all experiments, we set $\delta =0.0001$.

\begin{figure}
    \centering
    \includegraphics[width=0.485\textwidth]{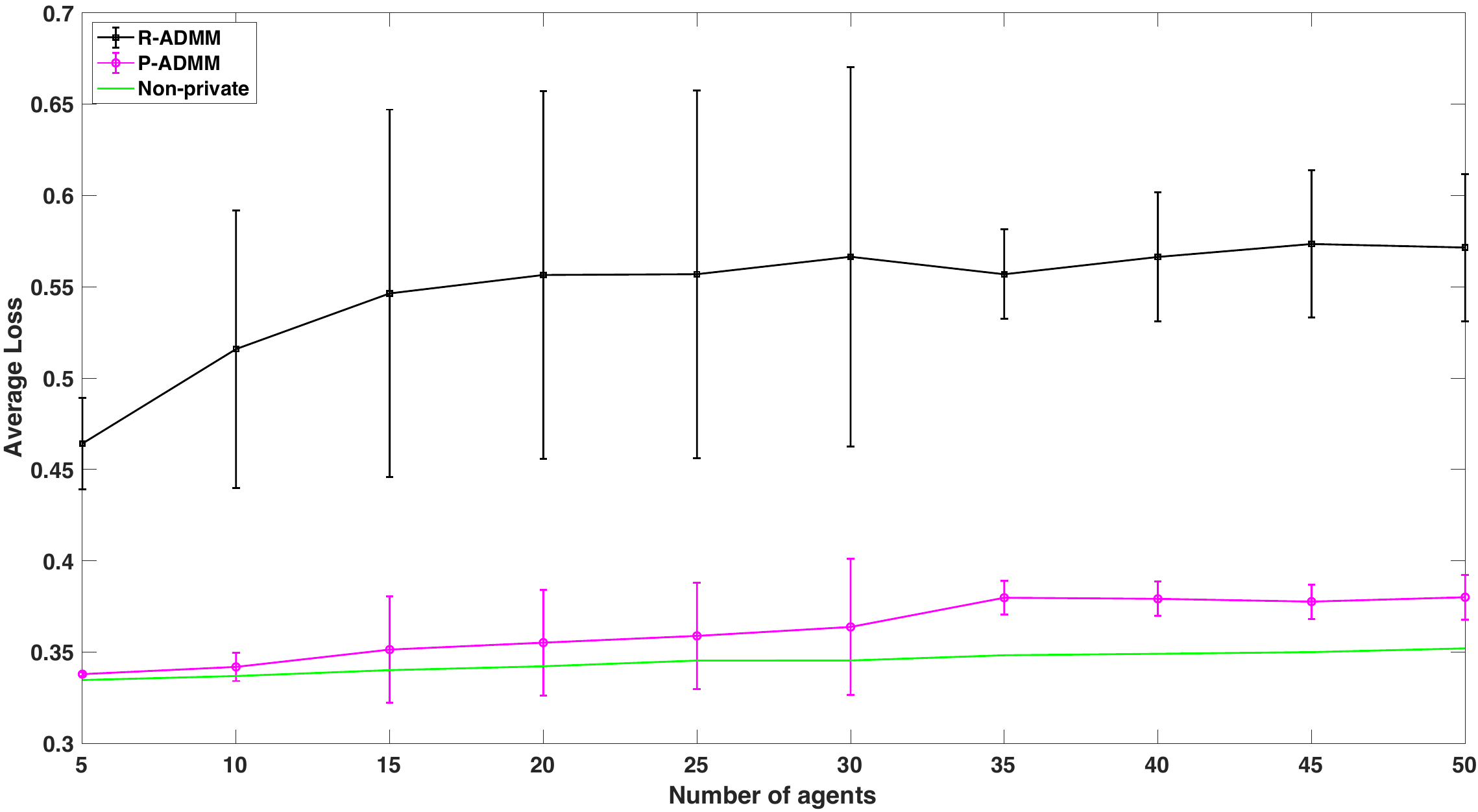}
    \caption{Compare accuracy for different agents: Total privacy loss $\epsilon = 10$.}
    \label{fig:diffagents}
\end{figure}

We present the results of our experiments with $L_2$-regularized logistic regression for a small network ($N=5$) in Figure \ref{fig_rg_5} and Figure \ref{fig_rg_10}. In particular, we first study the effect of varying the value of different variance decrease rate $R$ on the convergence of P-ADMM. For the total privacy loss $\epsilon = 5$ (Figure \ref{fig_rg_5}) and $\epsilon = 10$ (Figure \ref{fig_rg_10}), we can see that the variance decrease rate $R = 0.995$ results in a slower rate of privacy loss accumulation compared to benchmark algorithms and has quite low average loss during the
iterative process. In the right of Figure \ref{fig_rg_5} and Figure \ref{fig_rg_10}, P-ADMM has the lowest average loss compared to other differentially private algorithms. By quantifying the average loss versus Non-private, P-ADMM provides nearly the same convergence properties during the entire process. Moreover, 
we set the variance decrease rate $R = 0.995$ for the experiments with unregularized logistic regression on a small network ($N=5$), and present the results in Figure \ref{fig3.11}. As shown in Figure \ref{fig3.11}, P-ADMM outperforms the benchmark algorithms significantly in the small (a) and large (b) total privacy loss.

We also evaluate the effect of different number of agents by performing experiments with $L_2$-regularized logistic regression for networks of different sizes. Note that we control the total privacy loss to be the same during the entire iterative process ($K=50$) and varying the number of agents to compare accuracy (i.e. average loss at the final iteration). The results of this experiment are depicted in Figure \ref{fig:diffagents}. It is clear that our algorithm continues to outperform R-ADMM significantly for large networks.

\section{Conclusion}\label{conclusion}
In this paper, we have introduced a differentially private ADMM algorithm (P-ADMM) to address privacy concerns during the iterative process. In particular, we introduced Gaussian noise with linearly decaying variance in each iteration to preserves dynamic zCDP, a novel privacy framework that enjoys significant accuracy and tight privacy loss composition. 
We also theoretically analyzed convergence properties of the proposed algorithm for general convex and strongly convex optimization objectives.  
By performing extensive empirical comparisons with state-of-the-art methods for differentially private ADMM based machine learning algorithms, we demonstrated that P-ADMM exhibits superior convergence performance while providing strong privacy guarantee. 

\appendix
\begin{figure*}[t]
	\normalsize
	\begin{eqnarray}\label{eq:ii}
     &(1-b)\bigl( \frac{1}{4} - \frac{\theta}{1+\zeta}\bigl) \phi_{min}^2 (L_+) \|\tilde{x}^{k+1} - x^*\|_2^2
     + \bigl( 1- \frac{4\theta}{1+\zeta}\bigl) \|r^{k+1} - r^*\|_2^2\nonumber\\
    & \leq   \frac{1/4 + \theta}{1+\zeta}  \phi_{max}^2(L_+) \|\tilde{x}^{k} - x^*\|_2^2 + \frac{1}{1+\zeta} \|r^k - r^*\|_2^2  + \bigl( \frac{P +1/(2\theta)}{1+\zeta} +b\bigl( \frac{1}{4} - \frac{\theta}{1+\zeta}\bigl) \phi_{min}^2 (L_+)(\kappa_4 -1) \bigl) \|\xi^{k+1}\|_2^2
	\end{eqnarray}
	\hrulefill
\begin{eqnarray}\label{eq:i1}
     &\bigl( \frac{1}{4} - \frac{\theta}{1+\zeta}\bigl) \phi_{min}^2 (L_+) \|\tilde{x}^{k+1} - x^*\|_2^2 + \bigl( 1- \frac{4\theta}{1+\zeta}\bigl) \|r^{k+1} - r^*\|_2^2 \nonumber\\
    & \leq   \frac{1/4 + \theta}{1+\zeta}  \phi_{max}^2(L_+) \|\tilde{x}^{k} - x^*\|_2^2 + \frac{1}{1+\zeta} \|r^k - r^*\|_2^2 +  \frac{P +1/(2\theta)}{1+\zeta}  \|\xi^{k+1}\|_2^2 +\frac{4\theta \phi_{max}^2(M) }{1+\zeta}\|   x^{k+1} - x^*  \|_2^2 
	\end{eqnarray}
	\hrulefill
	\begin{eqnarray}\label{eq:i2}
		    &b\bigl( \frac{1}{4} - \frac{\theta}{1+\zeta}\bigl) \phi_{min}^2 (L_+)\|\tilde{x}^{k+1} - x^* \|_2^2 \nonumber\\
    &\geq b\bigl( \frac{1}{4} - \frac{\theta}{1+\zeta}\bigl) \phi_{min}^2 (L_+)(1-\frac{1}{\kappa_4}) \| x^{k+1} - x^*\|_2^2 -b\bigl( \frac{1}{4} - \frac{\theta}{1+\zeta}\bigl) \phi_{min}^2 (L_+)(\kappa_4 -1) \|\xi^{k+1}\|_2^2
    	\end{eqnarray}
	\hrulefill

\begin{eqnarray}\label{eq:i3}
&(1-b)\bigl( \frac{1}{4} - \frac{\theta}{1+\zeta}\bigl) \phi_{min}^2 (L_+) \|\tilde{x}^{k+1} - x^*\|_2^2 
+ \bigl( 1- \frac{4\theta}{1+\zeta}\bigl) \|r^{k+1} - r^*\|_2^2+ b\bigl( \frac{1}{4} - \frac{\theta}{1+\zeta}\bigl) \phi_{min}^2 (L_+)(1-\frac{1}{\kappa_4}) \| x^{k+1} - x^*\|_2^2 \nonumber\\
    & \leq   \frac{1/4 + \theta}{1+\zeta}  \phi_{max}^2(L_+) \|\tilde{x}^{k} - x^*\|_2^2 + \frac{1}{1+\zeta} \|r^k - r^*\|_2^2 + \bigl( \frac{P +1/(2\theta)}{1+\zeta} +b\bigl( \frac{1}{4} - \frac{\theta}{1+\zeta}\bigl) \phi_{min}^2 (L_+)(\kappa_4 -1) \bigl) \|\xi^{k+1}\|_2^2\nonumber\\ 
    & +\frac{4\theta \phi_{max}^2(M) }{1+\zeta}\|   x^{k+1} - x^*  \|_2^2   
    	\end{eqnarray}
	\hrulefill
\end{figure*}

\begin{lemma}\label{lemma6}
For $b\in(0,1)$ and $\kappa_4>1$, we have (\ref{eq:ii}) with $0 <\theta \leq \frac{b(1+ \zeta) \phi_{min}^2(L_+)(1- 1/\kappa_4)}{4b\phi_{min}^2(L_+)(1- 1/\kappa_4) + 16\phi_{max}^2(M) } $.
\end{lemma}
\begin{proof}
By decomposing $\|q^{k+1} - \hat{q}^* \|_G^2$ and $\|q^{k } - \hat{q}^* \|_G^2$, the result in Lemma~\ref{delta-coverge}  can be rewritten in the following form
\begin{align}\label{eq:it}
    &\|\frac{ \eta L_+ }{2  }(\tilde{x}^{k+1} - x^*)\|_2^2+\|\eta(r^{k+1}- r^*)\|_2^2    \nonumber\\
    & \leq\frac{1}{1+\zeta} \|\frac{ \eta L_+ }{2  }(\tilde{x}^{k+1} - x^*)\|_2^2 + \frac{1}{1+\zeta}\|\eta(r^{k+1}- r^*)\|_2^2\nonumber \\
    &~~~+ \frac{P}{1+\zeta} \|\xi^{k+1}\|_2^2+  
 \frac{1}{1+\zeta} \bigl< \xi^{k+1}, 
 \eta L_+ (\tilde{x}^{k+1}  
  - \tilde{x}^k)\nonumber\\
  &~~~+ 2\eta Q(r^{k+1}- r^*) 
 + 2\eta M (x^{k+1} - x^*) \bigl> \nonumber\\
 &\leq\frac{1}{1+\zeta} \|\frac{ \eta L_+ }{2  }(\tilde{x}^{k+1} - x^*)\|_2^2 + \frac{1}{1+\zeta}\|\eta(r^{k+1}- r^*)\|_2^2 \nonumber\\
 &~~~+ \bigl(\frac{P}{1+\zeta} + \frac{1/(2\theta)}{1+\zeta}\bigl) \|\xi^{k+1}\|_2^2 +\frac{\theta}{1+\zeta} \|\frac{\eta L_+}{2} (\tilde{x}^{k+1}  
  - \tilde{x}^k) \nonumber\\
 &~~~+  \eta Q(r^{k+1}- r^*) 
 +  \eta M (x^{k+1} - x^*) \|_2^2 \nonumber\\
  &\leq\frac{1}{1+\zeta} \|\frac{ \eta L_+ }{2  }(\tilde{x}^{k+1} - x^*)\|_2^2 + \frac{1}{1+\zeta}\|\eta(r^{k+1}- r^*)\|_2^2 \nonumber\\
 &~~~ +\frac{4\theta}{1+\zeta} \|\frac{\eta L_+}{2} (\tilde{x}^{k+1}   - x^*)\|_2^2   +\frac{4\theta}{1+\zeta} \|\frac{\eta L_+}{2} (\tilde{x}^{k }   - x^*)\|_2^2\nonumber \\
&~~~+\frac{4\theta}{1+\zeta}\|  \eta Q(r^{k+1}- r^*) \|_2^2 +\frac{4\theta}{1+\zeta}\| \eta M (x^{k+1} - x^*) \|_2^2 \nonumber\\
&~~~ +\bigl(\frac{P}{1+\zeta} + \frac{1/(2\theta)}{1+\zeta}\bigl) \|\xi^{k+1}\|_2^2.
\end{align}
In the second inequality, we use $\frac{1}{\theta}a^2 + \theta b^2 \geq 2ab$ for $\theta>0$. In the last inequality, we use $\| a+b\|_2^2 \leq 2\|a\|_2^2 +2\|b\|_2^2 $.

Rearranging the inequality (\ref{eq:it}) and diving by $\eta^2$, we have (\ref{eq:i1}), where we choose $\theta$ and $\zeta$ to satisfy $ \frac{1}{4} - \frac{\theta}{1+\zeta} >0$.
Since for any $\kappa>0$ and any matrices $C_1$ and $C_2$ with the same dimensions, there is 
\begin{align}
   \|C_1 + C_2\|_2^2 + (\kappa -1)\|C_1\|_2^2 \geq (1- \frac{1}{\kappa}) \|C_2\|_2^2, 
\end{align}
which implies 
$\|\tilde{x}^{k+1} - x^* \|_2^2  \geq (1-\frac{1}{\kappa_4}) \| x^{k+1} - x^*\|_2^2  -(\kappa_4 -1) \|\xi^{k+1}\|_2^2$ for any $\kappa_4 >1$.

Then, for $b\in(0,1)$, we obtain (\ref{eq:i2}).
Combining (\ref{eq:i1}) and (\ref{eq:i2}), we have (\ref{eq:i3}).

Next, if we set $\theta \leq \frac{b(1+ \zeta) \phi_{min}^2(L_+)(1- 1/\kappa_4)}{4b\phi_{min}^2(L_+)(1- 1/\kappa_4) + 16\phi_{max}^2(M) } $, it ensures 
\begin{align*}
    b\bigl( \frac{1}{4} - \frac{\theta}{1+\zeta}\bigl) \phi_{min}^2 (L_+)(1-\frac{1}{\kappa_4}) \geq \frac{4\theta \phi_{max}^2(M) }{1+\zeta}.
\end{align*}

Thus, the inequality (\ref{eq:i3}) can be written as (\ref{eq:ii}).
\end{proof}

\bibliography{./ding}
\bibliographystyle{IEEETran}

\end{document}